%% file: manuscript.v5.TEX
\documentclass[12pt]{article}
\usepackage{natbib}
\usepackage{graphicx,multirow,
setspace,amsmath,amssymb, mathrsfs,subfigure,url,color,algorithm,algcompatible}
\newenvironment{frcseries}{\fontfamily{frc}\selectfont}{}

\usepackage{epstopdf}
\usepackage{booktabs}
\input self-define-Heng.tex

\input self-define-HL.tex
\renewcommand{\log}{{\rm log}}

\newtheorem{thm}{Theorem}
\newtheorem{lem}{Lemma}

\newtheorem{rmk}{Remark}

\newcommand{\lan}{\left\langle}
\newcommand{\ran}{\right\rangle}

\newtheorem{assumption}{Assumption}

\DeclareMathOperator*{\argmin}{arg\,min}
\begin{document}

\title{Renewable Lasso without Batch-Number Constraints: A Gradient-Enhanced Approach}

\author{
Junzhuo Gao$^{a}$, Ling Peng$^{b,c}$, Xu Guo$^{d}$ and Heng Lian$^{e,a}$
\\
\begin{tabular}{l} 
{\small\it$^a$ Department of Mathematics, City University of Hong Kong, Hong Kong, China}\\
{\small\it$^b$ School of Statistics and Data Science, Jiangxi University of Finance and Economics,}\\
{\small\it $\;\; $ Nanchang, Jiangxi, China}\\
{\small\it$^c$ Philosophy and Social Sciences Laboratory of Data Science in Finance and Economics}\\
{\small\it $\;\; $ at the Ministry of Education, Jiangxi University of Finance and Economics,}\\
{\small\it $\;\; $ Nanchang, Jiangxi, China}\\
{\small\it$^d$ School of Statistics, Beijing Normal University, Beijing, China}\\
{\small\it$^e$ CityUHK Shenzhen Research Institute, Shenzhen, China}\\
\end{tabular}
}

\date{}          
\maketitle
\begin{abstract} We study online estimation for high-dimensional generalized linear models with streaming data.
First, for the non-distributed setting, we propose a gradient-enhanced surrogate loss that approximates the cumulative loss using only historical summaries, which modifies and improves upon the existing renewable estimation approach for the same model in the high-dimensional setting, and removes the batch-number constraint in previous studies. We then extend the method to distributed streaming data under the master-client architecture, where batches are partitioned across sites and only summaries (gradient vectors) are exchanged. Instead of directing applying the popular method of \cite{jordan2019communication} to the surrogate quadratic loss, our adjusted approach does not require the clients to compute the full surrogate loss. We derive non-asymptotic error bounds under the high-dimensional scaling, without the stringent constraint on the number of batches in the previous studies. Simulation results under linear and logistic models, together with a real-data application, show improved accuracy over existing renewable estimators.

\end{abstract}

\section{Introduction}\label{sec:intro}

\subsection{Background and literature review}

Streaming data, which arrive continuously in time and are often processed in batches, are now ubiquitous in modern applications such as online platforms, financial transactions, sensor networks, and medical monitoring systems.
In these settings, data volume grows rapidly and decisions must be updated sequentially as new batches arrive, while previous batches may already have been removed from storage.

These challenges naturally motivate online estimation methods, whose goal is to update model parameters sequentially when a new batch arrives, while avoiding storage and repeated access to the full historical raw data \citep{schifano2016online,wang2018online}.
Compared with classical offline methods, online procedures are designed to balance statistical accuracy, memory usage, and update latency.
  Stochastic gradient descent \citep[SGD]{robbins1951stochastic} is a computationally efficient and scalable optimization method, and is widely used in large-scale online learning.
A rich family of variants has since been developed, using momentum  \citep{polyak1964some}, Nesterov's accelerated gradient \citep{nesterov1983method}, and adaptive learning rate (AdaGrad \citep{duchi2011adaptive}; Adam \citep{kingma2014adam}).
However, SGD-type methods can be sensitive to step-size specification, and convergence may become unstable when the learning rate is misspecified.

To deal with the above-mentioned challenges, \cite{LuoSong2020renewable} developed an online estimation framework for generalized linear models, which they call renewable estimation.
Within this framework, likelihood-based estimators are updated using current-batch data and summaries of previous batches, without revisiting historical raw data.
Subsequent extensions to renewable estimation include  \citep{pan2024renewable,luo2023statistical,luo2023real}.
In particular, high-dimensional renewable estimation has been investigated \citep{luo2023online,  
xie2025statistical,rao2025estimation}.

Beyond temporal streaming, modern data pipelines are often distributed: data are generated and stored across hospitals, branch networks, mobile devices, and edge servers, making full centralization of raw records often impractical.
This architecture calls for collaborative learning procedures that update model parameters across sites while respecting locality constraints on data storage and computation.
Communication-efficient estimation has therefore received substantial attention for estimation and inference tasks  \citep{jordan2019communication, zhao2020debiasing, wang2020communication, fan2021communication, TanBatteyZhou2022commqr,duan2022heterogeneity}.

\subsection{Motivation and contributions}

In this paper, we study high-dimensional generalized linear models under an online, possibly also distributed, setting, where observations arrive sequentially in batches across multiple sites and only some summary statistics can be transmitted.
In many practical multi-site systems, privacy regulations, storage costs, and network bandwidth constraints prevent direct transfer of raw data to a central server.
Accordingly, the central methodological challenge is to design update rules that rely only on historical summaries while keeping statistical performance close to that of an ideal centralized estimator.

The contributions of this paper are twofold.
(1) First, for high-dimensional streaming generalized linear models, we revisit the formulation in \cite{luo2023online} and improve its approximation to the cumulative loss.
Methodologically, we introduce a gradient-enhanced surrogate loss that increases approximation accuracy with only a modest increase in stored summary statistics.
This estimator can be viewed as a simple refinement of \cite{luo2023online}. Using a direct Taylor expansion, we obtain a more accurate surrogate approximation to the cumulative loss. 
Theoretically, this makes it possible to remove the stringent restriction on the number of batches (in the high-dimensional setting) in the existing studies.
Empirically, simulation results show that the proposed estimator achieves higher statistical accuracy.
(2) Second, we extend this framework to distributed online settings, where data are partitioned across sites and only summary statistics can be exchanged.
Methodologically, we develop a communication-efficient distributed online estimator based on an approximation to the cumulative loss, requiring transmission of gradients only. In the distributed setting, our approach is closely related to but different from the framework of \cite{jordan2019communication}, which is designed for a fixed distributed sample. Our method is tailored to renewable estimation using explicit plug-in Hessian matrices on the master machine, which is natural in the online setting since Hessian matrices are already required for loss approximation in the non-distributed case. It is worth noting that the proposed method requires computing and storing the Hessian matrices on the master machine only. 


\subsection{Notations and organization}

For a vector $\va:=(a_1,\ldots,a_p)\trans\in\mathbb R^p$, we take $\|\va\|_r:=(|a_1|^r+\ldots+|a_p|^r)^{1/r}$ for $1\le r<\infty$ ($\|\va\|_2$ is simply denoted as $\|\va\|$), $\|\va\|_0:=\sum_{j=1}^p\mathbf1\{a_j\neq0\}$, and $\|\va\|_{\infty}:=\max_{1\le j\le p}|a_j|$. For a matrix $\vA\in\mathbb R^{p\times p}$ with entries $a_{ij}$, $\|\vA\|_{\max}:=\max_{i,j\le p}|a_{ij}|$.
The sub-Gaussian norm of a random variable $X$ is defined as $\|X\|_{\psi_2}:=\inf\{t>0:\mathbb E[\exp\{X^2/t^2\}]\le2\}$ and the sub-Gaussian norm of a random vector $\vX$ is defined as $\|\vX\|_{\psi_2}:=\sup_{\|\vu\|_2=1}\|\lan\vu, \vX\ran\|_{\psi_2}$.
For two positive sequences $\{a_n\}$ and $\{b_n\}$, we write $a_n\lesssim b_n$ or, equivalently, $b_n\gtrsim a_n$ to mean that there exists a constant $ C>0$ such that $a_n/b_n\le C$ for all $n$, and $a_n\asymp b_n$ means $a_n\lesssim b_n$ and $b_n\lesssim a_n$ at the same time. Finally, $C$ denotes a generic positive constant that can take different values even on the same line.

This paper is organized as follows. 
Section~\ref{sec:online} introduces the gradient-enhanced renewable estimator and establishes its convergence rates, emphasizing the difference from the previous study on the same model. 
Section~\ref{sec:dist-online} extends the method to distributed online settings, demonstrating that the same rates as the non-distributed case can be achieved under mild assumptions. 
Sections~\ref{sec:num} and \ref{sec:dataapp} report some simulation studies and a real data application, respectively. 
Section~\ref{sec:ext} concludes the paper with some discussions. 
All technical proofs are relegated to the appendix.

\section{Gradient-enhanced renewable estimation}\label{sec:online}

We assume the conditional distribution of the response $y$ comes from an exponential family with density $f(y;\theta,\phi)=b(y,\phi)\exp\{(y\theta-g(\theta))/\phi\}$, with $\theta=\vx\trans\vbeta_0$ and $\vbeta_0$ denotes the true parameter in the generalized linear model (GLM). $b$ and $g$ are known functions and we assume the nuisance parameter $\phi=1$ without loss of much generality since our focus is on estimating $\vbeta_0$. We note that $g(\theta+t)-g(\theta)$ (as a function of $t$) is actually the cumulant generating function for the distribution. 
Let $(\vx_i,y_i)\in\mathbb R^p\times\mathbb R$ be i.i.d. copies of  $(\vx,y)$.
The (per-sample) loss is constructed as the negative log-likelihood $\ell_i(\vbeta)=g(\vx_i\trans\vbeta)-y_i\vx_i\trans\vbeta$.

In a single-machine streaming data setting, observations arrive sequentially in batches. Let $\calD^{(b)}$ denote the $b$-th batch data with size $n^{(b)}:=|\calD^{(b)}|$ and let $N^{(b)}:=\sum_{j=1}^b n^{(j)}$ be the cumulative sample size after observing $b$ batches.
The cumulative loss up to the $b$-th time point is given by $\bar\calL^{(b)}(\vbeta)=\sum_{j=1}^b\calL^{(j)}(\vbeta)/N^{(b)}$, where $\calL^{(j)}(\vbeta):=\sum_{i\in\calD^{(j)}}\ell_i(\vbeta)$.
We consider a high-dimensional regime where $p$ may be comparable to, or larger than, the cumulative sample size $N^{(b)}$. 
We also assume the true parameter $\vbeta_0\in\mathbb R^p$ has sparsity $\|\vbeta_0\|_0=s^*$, and $\ell_1$-regularization is used to stabilize estimation and exploit the sparse structure.
The standard \emph{offline} lasso estimator is
\be\label{eqn:offline}
\hat\vbeta^{(b)}=\argmin_{\boldsymbol\beta\in\mathbb{R}^p}\{\bar\calL^{(b)}(\vbeta)+\lambda^{(b)}\|\vbeta\|_1\},
\ee 
where $\lambda^{(b)}>0$ is the regularization parameter. 

With streaming data for which raw observations from earlier batches are unavailable due to system constraints, historical loss terms cannot be recomputed directly and thus \eqref{eqn:offline} is infeasible.
An online estimation method aims to update an estimator when $\calD^{(b)}$ arrives without revisiting the raw data in $\calD^{(1)},\ldots,\calD^{(b-1)}$.
For the first batch $b=1$, 
\be\label{eqn:b1}
\breve\vbeta^{(1)}=\argmin_{\boldsymbol\beta\in\mathbb{R}^p}\{\bar\calL^{(1)}(\vbeta)+\lambda^{(1)}\|\vbeta\|_1\}
\ee
is obtained using the standard lasso, which is the same as \eqref{eqn:offline}.
For $b\ge 2$, the cumulative loss can be decomposed as $\bar\calL^{(b)}(\vbeta)=\calL^{(b)}(\vbeta)/N^{(b)}+\sum_{j=1}^{b-1}\calL^{(j)}(\vbeta)/N^{(b)}$, where the first term is directly available, while the rest must be approximated using historical batch-level summaries (e.g., gradient vectors and Hessian matrices).

We use a straightforward quadratic approximation based on Taylor's expansion to synthesize historical information as 
\begin{align*}
    \sum_{j=1}^{b-1}\calL^{(j)}(\vbeta)\approx&\ \sum_{j=1}^{b-1}\Big\{\calL^{(j)}(\breve\vbeta^{(j)})+(\vbeta-\breve\vbeta^{(j)})\trans\nabla\calL^{(j)}(\breve\vbeta^{(j)})\notag\\
    &\ \ \ \ \ \ \ \ +\frac{1}{2}(\vbeta-\breve\vbeta^{(j)})\trans\nabla^2\calL^{(j)}(\breve\vbeta^{(j)})(\vbeta-\breve\vbeta^{(j)})\Big\},
\end{align*}
where $\nabla\calL^{(j)}(\vbeta)=\sum_{i\in\calD^{(j)}}\{g'(\vx_i\trans\vbeta)-y_i\}\vx_i$ and $\nabla^2\calL^{(j)}(\vbeta)=\sum_{i\in\calD^{(j)}}g''(\vx_i\trans\vbeta)\vx_i\vx_i\trans$.
This leads to the surrogate loss 
\be\label{eq:online-l1}
\breve\calL^{(b)}(\vbeta)&=&\ \frac{1}{N^{(b)}}\Big\{\calL^{(b)}(\vbeta)+\vbeta\trans\sum_{j=1}^{b-1}\nabla\calL^{(j)}(\breve\vbeta^{(j)})\nonumber\\
	&&\ \ \ \ \ \ \ \ \ +\frac{1}{2}\sum_{j=1}^{b-1}(\vbeta-\breve\vbeta^{(j)})\trans\nabla^2\calL^{(j)}(\breve\vbeta^{(j)})(\vbeta-\breve\vbeta^{(j)})\Big\},
\ee 
and then we define 
\bse\breve\vbeta^{(b)}=\argmin_{\boldsymbol\beta\in\mathbb{R}^p}\{\breve\calL^{(b)}(\vbeta)+\lambda^{(b)}\|\vbeta\|_1\}.
\ese
The expression \eqref{eq:online-l1} above can be written in a slightly different form as
\be\label{eq:online-l2}
\breve\calL^{(b)}(\vbeta)&=&    \frac{1}{N^{(b)}}\Big\{\calL^{(b)}(\vbeta)+\vbeta\trans\sum_{j=1}^{b-1}\left\{ \nabla\calL^{(j)}(\breve\vbeta^{(j)})+\nabla^2\calL^{(j)}(\breve\vbeta^{(j)})(\breve\vbeta^{(b-1)}-\breve\vbeta^{(j)})\right\}\nonumber\\
	&&\ \ \ \ \ \ \ \ \ +\frac{1}{2}\sum_{j=1}^{b-1}(\vbeta-\breve\vbeta^{(b-1)})\trans\nabla^2\calL^{(j)}(\breve\vbeta^{(j)})(\vbeta-\breve\vbeta^{(b-1)})\Big\}+{\rm const},
\ee 
where $`{\rm const}'$ indicates a term that does not involve $\vbeta$.

Using \eqref{eq:online-l1} with $b$ replaced by $b-1$, it is easy to see that $\breve\vbeta^{(b-1)}$ satisfies the first-order optimality condition
\be\label{eqn:kkt}
 \sum_{j=1}^{b-1}\left\{ \nabla\calL^{(j)}(\breve\vbeta^{(j)})+\nabla^2\calL^{(j)}(\breve\vbeta^{(j)})(\breve\vbeta^{(b-1)}-\breve\vbeta^{(j)})\right\}+\lambda^{(b-1)} N^{(b-1)}\vxi=0,
\ee
with some $\vxi$ satisfying $\|\vxi\|_\infty\le 1$.
Thus, for the unpenalized case ($\lambda^{(b-1)}=0$ in \eqref{eqn:kkt}), which is reasonable only for low-dimensional models,  by plugging \eqref{eqn:kkt} into \eqref{eq:online-l2}, one can equivalently use the surrogate loss
\be\label{eq:online-l3}
\breve\calL^{(b)}(\vbeta)&=&    \frac{1}{N^{(b)}}\Big\{\calL^{(b)}(\vbeta)+\frac{1}{2}\sum_{j=1}^{b-1}(\vbeta-\breve\vbeta^{(b-1)})\trans\nabla^2\calL^{(j)}(\breve\vbeta^{(j)})(\vbeta-\breve\vbeta^{(b-1)})\Big\},
\ee 

For the penalized GLM, \cite{luo2023online} argued that  $$\|\sum_{j=1}^{b-1}\left\{ \nabla\calL^{(j)}(\breve\vbeta^{(j)})+\nabla^2\calL^{(j)}(\breve\vbeta^{(j)})(\breve\vbeta^{(b-1)}-\breve\vbeta^{(j)})\right\}\|_\infty=O(\lambda^{(b-1)} N^{(b-1)})$$ is small and can be ignored, and proposed to still use \eqref{eq:online-l3} in the penalized case. This strategy is followed by all subsequent studies on high-dimensional renewable estimation \citep{HanXieLiuSunHuangJiangKong2024,xie2025statistical,rao2025estimation}. However, upon reexamining how \eqref{eq:online-l3} was derived as we reviewed above, such an argument for ignoring the gradient term looks ungrounded. Technically, we will argue that use of \eqref{eq:online-l3} leads to slower rates which in turn leads to the unreasonable requirement that the number of batches considered should satisfy $b=o(\log N^{(b)})$ in all the above-mentioned studies in order to achieve (nearly) optimal rates, although the number of batches can be quite large in their simulation studies. We will demonstrate that, by using \eqref{eq:online-l1}, such a constraint on the number of batches can be removed theoretically, and our simulation results also show numerical improvements empirically. 

In implementation, using \eqref{eq:online-l3} as in \cite{luo2023online} requires storing $\breve\vbeta^{(b-1)}$ and $\sum_{j=1}^{b-1}\nabla^2\calL^{(j)}(\breve\vbeta^{(j)})$, our method based on \eqref{eq:online-l1} requires storing $\sum_{j=1}^{b-1}\nabla\calL^{(j)}(\breve\vbeta^{(j)})$, $\sum_{j=1}^{b-1}\nabla^2\calL^{(j)}(\breve\vbeta^{(j)})\breve\vbeta^{(j)}$, and $\sum_{j=1}^{b-1}\nabla^2\calL^{(j)}(\breve\vbeta^{(j)})$.
This implies modest extra storage requirement since the dominating storage cost is associated with the matrix $\sum_{j=1}^{b-1}\nabla^2\calL^{(j)}(\breve\vbeta^{(j)})$.
We refer to the estimator $\breve\vbeta^{(b)}$ in \eqref{eq:online-l1} as the \emph{gradient-enhanced} online estimator.

To obtain $\breve\vbeta^{(b)}$ numerically, we adopt the iterative shrinkage and thresholding algorithm (ISTA), which updates the estimates iteratively using
\begin{align}
    \breve\vbeta^{(b)}\leftarrow&\ \breve\vbeta^{(b)}-\frac{\eta}{N^{(b)}}\nabla\breve\calL^{(b)}(\breve\vbeta^{(b)})\label{eq:update_11}\\ \breve\beta_l^{(b)}\leftarrow&\ \mathcal{S}(\breve\beta_l^{(b)};\eta\lambda^{(b)}), l=1,\ldots,p,\label{eq:update_12}
\end{align}
where $\eta$ is the step size and $\calS(\beta;a):=\operatorname{sgn}(\beta)\max(|\beta|-a,0)$ is the soft-thresholding operator.
In the above, \eqref{eq:update_11} is a simple gradient descent step, and  \eqref{eq:update_12} uses soft-thresholding to produce a sparse solution.
In practice, the tuning parameter $\lambda^{(b)}$ is selected by a rolling-origin recalibration scheme as in \cite{luo2023online}. 
More specifically, as shown in Algorithm~\ref{alg:1}, upon the arrival of batch $b$, we use all previous batches to find the estimator and select the tuning parameter by computing the test error on the current batch. 

\begin{algorithm}[htbp]
    \setstretch{0.95}
    \caption{Gradient-enhanced online estimation}\label{alg:1}
    \begin{algorithmic}
        \STATE\textbf{Input:} batches $\calD^{(j)}$, $j=1,\ldots,B$; candidate grid $S_\lambda$ for the regularization parameter.
        \STATE\textbf{Output:} online estimators $\breve\vbeta^{(j)}$, $j=1,\ldots,B$.
        \STATE Compute $\breve\vbeta_{\lambda}^{(1)}$ for all $\lambda\in S_\lambda$, and set  $\breve\vbeta^{(1)}=\breve\vbeta_{\lambda^{(1)}}^{(1)}$ with $\lambda^{(1)}$ found using cross-validation.
        \FOR{$b=2,\ldots,B$}
    

                \STATE Compute $\breve\vbeta_{\lambda}^{(b)}$ for all $\lambda\in S_\lambda$.
                \STATE Compute $\mathrm{MSPE}(\lambda)=\sum_{i\in\calD^{(b)}}\{y_i-g'(\vx_i\trans\breve\vbeta_{\lambda}^{(b-1)})\}^2$ for all $\lambda\in S_\lambda$. 
                \STATE Set  $\breve\vbeta^{(b)}=\breve\vbeta_{\lambda^{(b)}}^{(b)}$ with $\lambda^{(b)}=\arg\min_{\lambda\in S_\lambda}\mathrm{MSPE}(\lambda)$.
         \ENDFOR
    \end{algorithmic}
\end{algorithm}

We derive non-asymptotic error bounds for the proposed renewable estimator using the following assumptions. 

\begin{assumption}\label{ass:design-noise}
    The covariate vector $\vx\in\mathbb R^p$ is sub-Gaussian with parameter $\sigma_x$.
    Let $\vSigma:=\mathbb E(\vx\vx\trans)$ and assume $\vSigma\succeq k_1 \vI_p$ for some $k_1>0$.
    Here $\sigma_x$ and $k_1$ are assumed to be constants.
\end{assumption}

\begin{assumption}\label{ass:smooth}
    The functions $g''$ and $g'''$ are bounded by a constant $L_g>0$.
    Moreover, on any fixed compact interval, $g''$ is bounded away from zero.
\end{assumption}

Assumption~\ref{ass:design-noise} specifies that the covariate vector is sub-Gaussian, while Assumption~\ref{ass:smooth} imposes smoothness conditions on the link function. Assumption~\ref{ass:smooth} does not hold for all generalized linear models, but it automatically holds for least squares regression and logistic regression. Boundedness of the derivatives of $g$, while may be relaxed, can greatly simplify the proofs and thus frequently adopted. 
These conditions are standard in the high-dimensional generalized linear model literature \citep{negahban2009unified,ning2017general,tian2023transfer}.

\begin{thm}\label{thm:online-lasso}
    Suppose Assumptions~\ref{ass:design-noise}--\ref{ass:smooth} hold.
    Let $s^*:=\|\vbeta_0\|_0$ and $\breve\vbeta^{(b)}, b\le B$ be defined in \eqref{eq:online-l1} with tuning parameter $\lambda^{(b)}=C\sqrt{\log(p\vee N^{(b)})/N^{(b)}}$ for a sufficiently large constant $C>0$.
    Assume that batch sizes satisfy $n^{(b)}\gtrsim s^*\log(p)$ for all $b\le B$, and that
    \begin{equation}\label{eq:scale1}
        \max_{2\le b\le B}\frac{s^*\log(p\vee N^{(b)})(1+\log(N^{(b-1)}/N^{(1)}))}{\sqrt{N^{(b)}}}=o(1).
    \end{equation}
	Then with probability at least 
    $1-\sum_{b=1}^B\left( \exp\{-C\log(p\vee N^{(b)})\}+\exp\{-Cn^{(b)}\}\right)$
	\begin{equation*}
	   \|\breve\vbeta^{(b)}-\vbeta_0\|_2\le C\sqrt{s^*}\,\lambda^{(b)},
	   \quad\|\breve\vbeta^{(b)}-\vbeta_0\|_1\le Cs^*\lambda^{(b)}, \quad\forall\,b\le B.
	\end{equation*}
\end{thm}

\begin{rmk}
    As mentioned before, an important difference between our theoretical result and Theorem 1 of \cite{luo2023online} is that our obtained rate is $C\sqrt{s^*}\,\lambda^{(b)}$, for example for the $\ell_2$ error, with a constant $C$. In particular, here $C$ is independent of batch number $b$. In \cite{luo2023online}, the rate is $C^b\sqrt{s^*}\,\lambda^{(b)}$ where the factor $C^b$ ($C$ raised to the power $b$) increases exponentially fast. Thus to put $C^b$ under control, $b$ can only be no larger than a constant (\cite{luo2023online} mentioned allowing $b=o(\log N^{(b)})$ so that $C^b$ increase slower than $(N^{(b)})^{\ep}$, $\forall \ep>0$). 

    Technically, the reason for this improvement is due to the use of direct quadratic approximation without ignoring the gradient as in \cite{luo2023online}. More specifically, the proof of Theorem \ref{thm:online-lasso} proceeds by considering $b=1,\ldots, B$ sequentially and tracking the effect of estimation error of $\breve\vbeta^{(j)}$, $j\le b-1$ on the error of $\breve\vbeta^{(b)}$.
    In Step 3 of the proof of Theorem \ref{thm:online-lasso}, due to the direct use of Taylor's expansion, it is easy to see the term $\vR^{(j)}$ defined in \eqref{eqn:R}, which shows the main effect of  $\breve\vbeta^{(j)}$, is proportional to $\|\breve\vbeta^{(j)}-\vbeta_0\|^2$. Under suitable assumptions, this would become a higher order term compared to the main stochastic error term. On the other hand, when some gradient terms are ignored, the effect of $\breve\vbeta^{(b-1)}$ would enter the bound as $\|\breve\vbeta^{(b-1)}-\vbeta_0\|$ (not squared) and such a term would make $\|\breve\vbeta^{(b)}-\vbeta_0\|$ proportional to $\|\breve\vbeta^{(b-1)}-\vbeta_0\|$ in the theoretical analysis, and the multiplicative factor in front of $\|\breve\vbeta^{(b-1)}-\vbeta_0\|$ makes the error bound increase exponentially fast.
    
\end{rmk}

\begin{rmk}
    Suppose all batches have the same size, so $n^{(b)}\equiv n^{(1)}$ and $N^{(b)}/N^{(1)}=b$.
    Then \eqref{eq:scale1} reduces to $\max_{2\le b\le B} s^*\max\{\log(b),\log(p),\log(n^{(1)})\}\,\log(b)/\sqrt{bn^{(1)}}=o(1)$, which is implied by $s^*\log(p\vee n^{(1)})/\sqrt{n^{(1)}}=o(1)$.
\end{rmk}

\section{Distributed online estimation}\label{sec:dist-online}
 We consider a distributed architecture with a total of $K$ machines, where the first machine is designated as the master (or called server), which can communicate with all other machines, called workers (or called clients). 
 The raw data should remain on local machines due to communication and storage constraints. The master updates the global estimator using only summary statistics transmitted from other machines (e.g., gradient vectors), rather than aggregating all raw observations.

At each time point $b$, the $k$-th machine receives a new local batch $\calD_k^{(b)}$ of size $n_k^{(b)}:=|\calD_k^{(b)}|$.
Define $N_K^{(b)}:=\sum_{k=1}^K\sum_{j=1}^b n_k^{(j)}$ and $\calL_k^{(b)}(\vbeta):=\sum_{i\in\mathcal D_k^{(b)}}\ell_i(\vbeta)$.
Then, reusing some notations in the non-distributed setting,  the cumulative loss is $\bar\calL^{(b)}(\vbeta)=\sum_{j=1}^b\sum_{k=1}^K\calL_k^{(j)}(\vbeta)/N_K^{(b)}$.
The corresponding centralized offline lasso estimator is $\hat\vbeta^{(b)}=\argmin_{\boldsymbol\beta\in\mathbb{R}^p}\{\bar\calL^{(b)}(\vbeta)+\lambda^{(b)}\|\vbeta\|_1\}$, where $\lambda^{(b)}>0$ is the regularization parameter.

A naive baseline method is the one-shot averaging estimator at each time point.
Specifically, each machine computes $\breve\vbeta_k^{(b)}$ as in Section~\ref{sec:online} and sends it to the master, which then forms $\bar\vbeta^{(b)}=\sum_{k=1}^K\sum_{j=1}^bn_k^{(j)}\breve\vbeta_k^{(b)}/N_K^{(b)}$. Although simple to implement, this one-shot estimator can be notably less efficient than the offline lasso estimator when the number of machines is large and the parameter is high-dimensional.

To overcome this issue, we propose a {\it communication-efficient} distributed online estimator.
 For $b=1$, given an initial estimator $\breve\vbeta^{(0)}$ (for concreteness, this would be the estimator in \eqref{eqn:b1} using data on the master machine), we construct a surrogate for the global loss as follows.
Specifically, we write
\begin{equation*}
    \bar\calL^{(1)}(\vbeta)=\frac{1}{N_K^{(1)}}\calL_1^{(1)}(\vbeta)+\frac{1}{N_K^{(1)}}\sum_{k=2}^K\calL_k^{(1)}(\vbeta),
\end{equation*}
and apply a second-order Taylor expansion to the non-master part $\sum_{k=2}^K\calL_k^{(1)}(\vbeta)$ around $\breve\vbeta^{(0)}$:
\begin{equation*}
    \sum_{k=2}^K\Big\{\calL_k^{(1)}(\breve\vbeta^{(0)})+(\vbeta-\breve\vbeta^{(0)})\trans\nabla\calL_k^{(1)}(\breve\vbeta^{(0)})+\frac{1}{2}(\vbeta-\breve\vbeta^{(0)})\trans\nabla^2\calL_k^{(1)}(\breve\vbeta^{(0)})(\vbeta-\breve\vbeta^{(0)})\Big\}.
\end{equation*}
Because transmitting $p\times p$ matrices $\nabla^2\calL_k^{(1)}(\breve\vbeta^{(0)})$, $k=1,\ldots,K$ from all machines to the master is costly, we use the approximation $\nabla^2\calL_k^{(1)}(\breve\vbeta^{(0)})\approx n_k^{(1)}\nabla^2\calL_1^{(1)}(\breve\vbeta^{(0)})/n_1^{(1)}$.
After discarding the terms independent of $\vbeta$, we define the surrogate loss
\begin{align}
	\breve\calL^{(1)}(\vbeta):=&\,\frac{1}{N_K^{(1)}}\Big\{\calL_1^{(1)}(\vbeta)+\vbeta\trans\sum_{k=2}^K\nabla\calL_k^{(1)}(\breve\vbeta^{(0)})\notag\\
	&\ \ \ \ \ \ \ +\frac{\sum_{k=2}^Kn_k^{(1)}}{2n_1^{(1)}}(\vbeta-\breve\vbeta^{(0)})\trans\nabla^2\calL_1^{(1)}(\breve\vbeta^{(0)})(\vbeta-\breve\vbeta^{(0)})\Big\},\label{eq:surrogate_b1}
\end{align}
and define $\breve\vbeta^{(1)}=\argmin_{\boldsymbol\beta\in\mathbb{R}^p}\{\breve\calL^{(1)}(\vbeta)+\lambda^{(1)}\|\vbeta\|_1\}$. Note that to construct the surrogate loss, the gradients $\nabla\calL_k^{(1)}(\breve\vbeta^{(0)})$ should be transmitted to the master for aggregation.

\begin{rmk}
	The surrogate loss above is closely related to but different from that proposed by \cite{jordan2019communication}.
	Indeed, we directly use $\nabla^2\calL_1^{(1)}(\breve\vbeta^{(0)})$ which is computed on the master, while this matrix can be regarded as implicitly computed using gradient difference in \cite{jordan2019communication}.
    To see the two approaches are closely related,
    we can use the approximation
	\begin{equation*}
		\frac{1}{2}(\vbeta-\breve\vbeta^{(0)})\trans\nabla^2\calL_1^{(1)}(\breve\vbeta^{(0)})(\vbeta-\breve\vbeta^{(0)})\approx\calL_1^{(1)}(\vbeta)-\calL_1^{(1)}(\breve\vbeta^{(0)})-(\vbeta-\breve\vbeta^{(0)})\trans\nabla\calL_1^{(1)}(\breve\vbeta^{(0)}),
	\end{equation*}
	and substitute this into \eqref{eq:surrogate_b1} to get
	\begin{equation}\label{eqn:jordanloss}
		\frac{1}{n_1^{(1)}}\calL_1^{(1)}(\vbeta)-\vbeta\trans\left\{\frac{1}{n_1^{(1)}}\nabla\calL_1^{(1)}(\breve\vbeta^{(0)})-\frac{1}{N_K^{(1)}}\sum_{k=1}^K\nabla\calL_k^{(1)}(\breve\vbeta^{(0)})\right\},
	\end{equation}
	which is exactly  the surrogate loss in \cite{jordan2019communication}.
	Compared to \eqref{eqn:jordanloss}, 
    the proposed \eqref{eq:surrogate_b1} appears more natural in the online setting, since direct quadratic approximation is already required even in the non-distributed setting.
\end{rmk}

For $b\ge 2$, we still approximate $\bar\calL^{(b)}(\vbeta)$ via a surrogate loss construction.
We decompose $\bar\calL^{(b)}(\vbeta)$ into three components as follows:
\begin{equation}\label{eq:three}
    \bar\calL^{(b)}(\vbeta)=\frac{1}{N_K^{(b)}}\calL_1^{(b)}(\vbeta)+\frac{1}{N_K^{(b)}}\sum_{k=2}^K\calL_k^{(b)}(\vbeta)+\frac{1}{N_K^{(b)}}\sum_{k=1}^K\sum_{j=1}^{b-1}\calL_k^{(j)}(\vbeta),
\end{equation}
where $\calL_1^{(b)}(\vbeta)$ corresponds to the new data at the master, $\sum_{k=2}^K\calL_k^{(b)}(\vbeta)$ corresponds to new data on non-master machines, and $\sum_{k=1}^K\sum_{j=1}^{b-1}\calL_k^{(j)}(\vbeta)$ aggregates historical data up to batch $b-1$ from all machines.
Let $\breve\vbeta^{(j)}, j=1,\ldots,b-1$ denote the proposed estimators up to batch $b-1$. 
Similarly, for $1\le j\le b-1$, with $\breve\vbeta^{(j)}$ available, we first approximate $\calL_k^{(j)}(\vbeta)$ in \eqref{eq:three} by
\begin{equation*}
    \calL_k^{(j)}(\breve\vbeta^{(j)})+(\vbeta-\breve\vbeta^{(j)})\trans\nabla\calL_k^{(j)}(\breve\vbeta^{(j)})+\frac{1}{2}(\vbeta-\breve\vbeta^{(j)})\trans\nabla^2\calL_k^{(j)}(\breve\vbeta^{(j)})(\vbeta-\breve\vbeta^{(j)}).
\end{equation*}
For $\calL_k^{(b)}(\vbeta)$ with $2\le k\le K$ in \eqref{eq:three}, we expand it at $\breve\vbeta^{(b-1)}$ as
\begin{equation*}
    \calL_k^{(b)}(\breve\vbeta^{(b-1)})+(\vbeta-\breve\vbeta^{(b-1)})\trans\nabla\calL_k^{(b)}(\breve\vbeta^{(b-1)})+\frac{1}{2}(\vbeta-\breve\vbeta^{(b-1)})\trans\nabla^2\calL_k^{(b)}(\breve\vbeta^{(b-1)})(\vbeta-\breve\vbeta^{(b-1)}).
\end{equation*}
Using the above approximations directly to construct the surrogate loss would require, for example, computing and transmitting $\nabla^2\calL_k^{(j)}(\vbeta)$ on all non-master machines. 
Thus, to further simplify the procedure, we approximate the Hessian matrices by the ones that can be computed on the master:
\begin{equation*}
    \nabla^2\calL_k^{(j)}(\breve\vbeta^{(j)})\approx\frac{n_k^{(j)}}{n_1^{(j)}}\nabla^2\calL_1^{(j)}(\breve\vbeta^{(j)}), \quad\nabla^2\calL_k^{(b)}(\breve\vbeta^{(b-1)})\approx\frac{n_k^{(b)}}{n_1^{(b)}}\nabla^2\calL_1^{(b)}(\breve\vbeta^{(b-1)}).
\end{equation*}
Based on these approximations, we finally can define the following surrogate loss function
\begin{align}\label{eqn:distsurr}
    \breve\calL^{(b)}(\vbeta):=&{}\frac{1}{N_K^{(b)}}\Big[\calL_1^{(b)}(\vbeta)+\vbeta\trans\Big\{\sum_{k=2}^K\nabla\calL_k^{(b)}(\breve\vbeta^{(b-1)})+\sum_{k=1}^K\sum_{j=1}^{b-1}\nabla\calL_k^{(j)}(\breve\vbeta^{(j)})\Big\}\nonumber\\
    &\ \ \ \ \ \ \ \ +\frac{\sum_{k=2}^Kn_k^{(b)}}{2n_1^{(b)}}(\vbeta-\breve\vbeta^{(b-1)})\trans\nabla^2\calL_1^{(b)}(\breve\vbeta^{(b-1)})(\vbeta-\breve\vbeta^{(b-1)})\nonumber\\
    &\ \ \ \ \ \ \ \ +\sum_{j=1}^{b-1}\frac{\sum_{k=1}^Kn_k^{(j)}}{2n_1^{(j)}}(\vbeta-\breve\vbeta^{(j)})\trans\nabla^2\calL_1^{(j)}(\breve\vbeta^{(j)})(\vbeta-\breve\vbeta^{(j)})\Big].
\end{align}

To compute $\breve\vbeta^{(b)}=\argmin_{\boldsymbol\beta\in\mathbb{R}^p}\{\breve\calL^{(b)}(\vbeta)+\lambda^{(b)}\|\vbeta\|_1\}$ for $b\ge2$, we again use ISTA with the updates
\begin{equation*}
    \breve\vbeta^{(b)}\leftarrow\breve\vbeta^{(b)}-\eta\nabla\breve\calL^{(b)}(\breve\vbeta^{(b)}),~\breve\beta_l^{(b)}\leftarrow\mathcal{S}(\breve\beta_l^{(b)};\eta\lambda^{(b)}), l=1,\ldots,p,
\end{equation*}
where
\begin{align*}
    &\ \nabla\breve\calL^{(b)}(\vbeta)=\frac{1}{N_K^{(b)}}\Big\{\nabla\calL_1^{(b)}(\vbeta)+\sum_{k=2}^K\nabla\calL_k^{(b)}(\breve\vbeta^{(b-1)})+\sum_{k=1}^K\sum_{j=1}^{b-1}\nabla\calL_k^{(j)}(\breve\vbeta^{(j)})\\
    &\ +\frac{\sum_{k=2}^Kn_k^{(b)}}{n_1^{(b)}}\nabla^2\calL_1^{(b)}(\breve\vbeta^{(b-1)})(\vbeta-\breve\vbeta^{(b-1)})+\sum_{j=1}^{b-1}\frac{\sum_{k=1}^Kn_k^{(j)}}{n_1^{(j)}}\nabla^2\calL_1^{(j)}(\breve\vbeta^{(j)})(\vbeta-\breve\vbeta^{(j)})\Big\}.
\end{align*}
\begin{rmk}
For implementation, the master needs to store $\sum_{j=1}^{b-1}\sum_{k=1}^K\nabla\calL_k^{(j)}(\breve\vbeta^{(j)})$, $\sum_{j=1}^{b-1}(\sum_{k=1}^Kn_k^{(j)}/n_1^{(j)})\nabla^2\calL_1^{(j)}(\breve\vbeta^{(j)})$, and $\sum_{j=1}^{b-1}(\sum_{k=1}^Kn_k^{(j)}/n_1^{(j)})\nabla^2\calL_1^{(j)}(\breve\vbeta^{(j)})\breve\vbeta^{(j)}$, all of which can obviously be updated sequentially. 
In terms of communication, the master needs to obtain $\nabla\calL_k^{(b)}(\breve\vbeta^{(b-1)})$ and $n_k^{(b)}$ from other machines, without the need for Hessian matrices from other machines.
Consequently, the resulting algorithm is deemed communication-efficient.
\end{rmk}
\begin{rmk}The tuning parameter $\lambda^{(b)}$ can also be determined by the ``rolling-origin recalibration'' procedure.
As an extension to the non-distributed case, the master also needs to aggregate test errors received from other machines.
We summarize the distributed online estimation procedure in Algorithm~\ref{alg:2}.
\end{rmk}
\begin{rmk}
Alternative to the surrogate loss defined in \eqref{eqn:distsurr}, one can also directly apply the approach of \cite{jordan2019communication}, after approximating the loss on each machine by the surrogate loss \eqref{eq:online-l1}. However, this would force each non-master machine to compute and store $\nabla^2\calL_k^{(j)}(\breve\vbeta^{(j)})$, which is not required in our approach when only Hessian matrices based on the data on the master are utilized. 
\end{rmk}

\begin{algorithm}[htbp]
    \caption{Distributed online estimation}\label{alg:2}
    \begin{algorithmic}
        \STATE\textbf{Input:} distributed batches $\calD_k^{(j)}$ ($k=1,\ldots,K$, $j=1,\ldots,B$); candidate grid $S_\lambda$.
        \STATE\textbf{Output:} estimators $\breve\vbeta^{(j)}$, $j=1,\ldots,B$.
                \STATE The master computes $\breve\vbeta^{(0)}$ based on data $\calD_1^{(1)}$ and broadcasts it to all machines.
                \STATE Each machine evaluates $\nabla\calL_k^{(1)}(\breve\vbeta^{(0)})$ and sends it to the master.
                \STATE The master computes $\breve\vbeta_{\lambda}^{(1)}$ for all $\lambda\in S_\lambda$.
                \STATE The master chooses $\lambda^{(1)}$ that  minimizes cross-validation error and  set $\breve\vbeta^{(1)}=\breve\vbeta_{\lambda^{(1)}}^{(1)}$.
                \STATE The master broadcasts $\breve\vbeta^{(1)}$ to all machines and receives $\nabla\calL_k^{(1)}(\breve\vbeta^{(1)})$ from them.
           \FOR{$b=2,\ldots,B$}
                \STATE Each machine evaluates $\nabla\calL_k^{(b)}(\breve\vbeta^{(b-1)})$ and sends it to the master.
                \STATE The master computes $\breve\vbeta_{\lambda}^{(b)}$ for all $\lambda\in S_\lambda$, based on the surrogate loss.
                \STATE The master aggregates the test error (requiring additional communication): $$\mathrm{MSPE}(\lambda)=\sum_{k=1}^K\sum_{i\in\calD_k^{(b)}}\{y_i-g'(\vx_i\trans\breve\vbeta_{\lambda_l}^{(b-1)})\}^2.$$
                \STATE The master sets $\lambda^{(b)}=\arg\min_{\lambda\in S_\lambda}\mathrm{MSPE}(\lambda)$ and $\breve\vbeta^{(b)}=\vbeta_{\lambda^{(b)}}^{(b)}$. 
                \STATE The master broadcasts $\breve\vbeta^{(b)}$ to all others and receives $\nabla\calL_k^{(b)}(\breve\vbeta^{(b)})$ from them.            
        \ENDFOR
    \end{algorithmic}
\end{algorithm}

Finally, non-asymptotic error bounds for the distributed renewable lasso estimator are presented in Theorem \ref{thm:dist-batch}.
\begin{thm}\label{thm:dist-batch}
    Suppose Assumptions~\ref{ass:design-noise}--\ref{ass:smooth} hold. Let $\breve\vbeta^{(b)}$, $1\le b\le B$, be defined by the distributed surrogate program with tuning parameter $\lambda^{(b)}=C\sqrt{\log(p\vee N_K^{(b)})/N_K^{(b)}}$ for a sufficiently large constant $C>0$ and initial estimator $\breve\vbeta^{(0)}$ (obtained using data $\calD_1^{(1)}$ with tuning parameter $\lambda^{(0)}=C\sqrt{\log (p\vee n_1^{(1)})/n_1^{(1)}}$). Assume further that $n_k^{(b)}\gtrsim s^*\log(p)$ for all $1\le b\le B$ and $1\le k\le K$, and also that $s^*\log(p\vee n_1^{(1)})(N_K^{(1)}-n_1^{(1)})/(n_1^{(1)}\sqrt{N_K^{(1)}})=o(1)$, 
    \begin{equation}\label{eq:scale2}
        \max_{2\le b\le B}\frac{s^*\log(p\vee N_K^{(b)})}{\sqrt{N_K^{(b)}}}\left\{\frac{N_K^{(b)}-N_K^{(b-1)}-n_1^{(b)}}{N_K^{(b-1)}}+1+\log\Big(\frac{N_K^{(b-1)}}{N_K^{(1)}}\Big)\right\}=o(1).
    \end{equation}
    Then, with probability at least 
	$1-\sum_{b=1}^B\left(\exp\{-C\log(p\vee n_1^{(b)})\}+\exp\{-C\log(p\vee \sum_{k=2}^Kn_k^{(b)})\}\right),$
	the following bounds hold simultaneously for all $1\le b\le B$:
    \begin{equation*}
        \|\breve\vbeta^{(b)}-\vbeta_0\|_2\le C\sqrt{s^*}\,\lambda^{(b)},
        \quad\|\breve\vbeta^{(b)}-\vbeta_0\|_1\le Cs^*\lambda^{(b)}.
    \end{equation*}
\end{thm}

\begin{rmk}
    Theorem~\ref{thm:dist-batch} extends Theorem~\ref{thm:online-lasso} to the distributed setting. Compared with the single-site case, the analysis must additionally account for the number of sites $K$ and the approximation error induced by replacing raw data with  summary statistics from local machines. The condition $s^*\log(p\vee n_1^{(1)})(N_K^{(1)}-n_1^{(1)})/(n_1^{(1)}\sqrt{N_K^{(1)}})=o(1)$ is used to derive the error bound for $\breve\vbeta^{(1)}$. The condition in \eqref{eq:scale2} depends explicitly on both the number of sites and the number of batches. If $n_k^{(j)}\equiv n_1^{(1)}$, then \eqref{eq:scale2} simplifies to $\max_{2\le b\le B}s^*\max\{\log(b),\log(K),\log(p),\log(n_1^{(1)})\}\,\log(b-1)/\sqrt{bKn_1^{(1)}}=o(1)$, which is in turn implied by $s^*\log(p\vee n_1^{(1)})/\sqrt{n_1^{(1)}}=o(1)$. When $K=1$, \eqref{eq:scale2} reduces to \eqref{eq:scale1}.
\end{rmk}

\section{Simulation studies}\label{sec:num}

This section evaluates the finite-sample performance of the proposed methods in both non-distributed and distributed streaming data settings.
We focus on two representative models (linear and logistic) and report the mean squared errors (MSE) across Monte Carlo replications.

\subsection{Online estimation}\label{sec:num1}

We first consider the non-distributed streaming data setting.
The data-generating mechanism is as follows.
(i) Linear model $y_i=\vx_i\trans\vbeta+\epsilon_i$.
The covariates $\vx_i$ are generated from a multivariate normal distribution with mean zero and covariance matrix $\vSigma_{ij}=0.5^{|i-j|}$.
We set $\vbeta=(1,1,1,1,1,0,\ldots,0)\trans$, and generate $\epsilon_i$ from either $N(0,1)$ or a $t$-distribution with $5$ degrees of freedom.
(ii) Logistic model $\log\left(\frac{\pi_i}{1-\pi_i}\right)=\vx_i\trans\vbeta$, $\pi_i:=\mathbb E[y_i\mid\vx_i]$.
The covariates $\vx_i$ are generated from a multivariate normal distribution with covariance $\vSigma_{ij}=0.3^{|i-j|}$ or $\vSigma_{ij}=0.5^{|i-j|}$.
We set $\vbeta=(0.5,0.5,0.5,0.5,0.5,0,\ldots,0)\trans$.

We use $n^{(b)}=80$ for all $b=1,\ldots,B$, with $B=50$, $p\in\{250,500\}$, and $200$ replications.
We compare the proposed gradient-enhanced online estimator (\texttt {online}) with the full offline estimator (\texttt{offline}), which has access to all historical data, and the online method in \cite{luo2023online} (\texttt{luo}).
Figures \ref{fig:non_dist_lin} and \ref{fig:non_dist_log} report $\log(\mathrm{MSE})$ under the linear and logistic models, respectively.

The main observations are as follows. 
(i) For all methods, $\log(\mathrm{MSE})$ decreases as the batch index $b$ increases.
(ii) For a fixed $b$, increasing the dimension from $p=250$ to $p=500$ leads to an increase in errors.
(iii) The proposed \texttt{online} method is consistently closer to the offline benchmark, compared with the \texttt{luo} method.

\begin{figure}[htbp]
    \centering
    \includegraphics[width=0.6\linewidth]{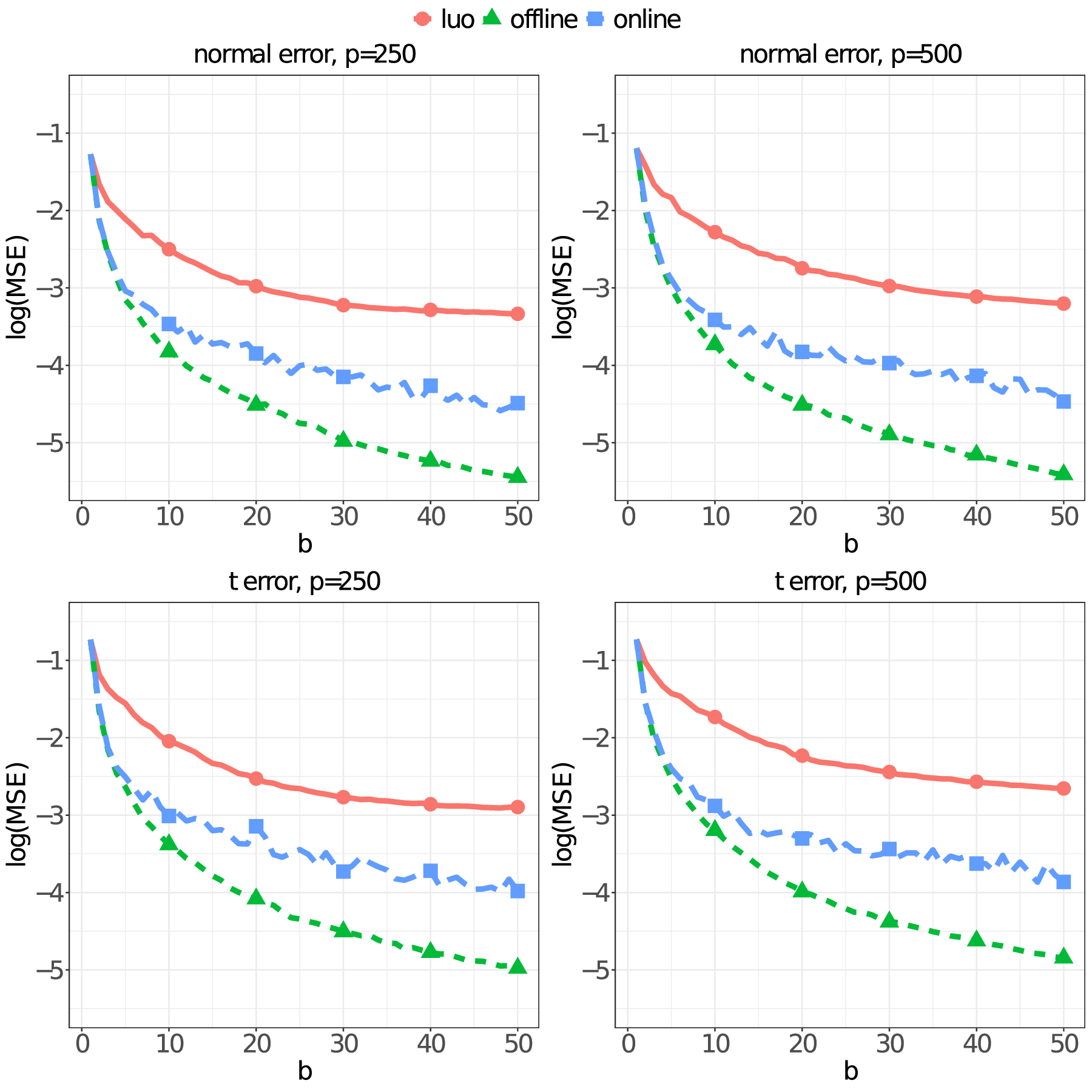}
    \caption{Comparison of $\log(\mathrm{MSE})$ under the linear model.}
    \label{fig:non_dist_lin}
\end{figure}

\begin{figure}[htbp]
    \centering
    \includegraphics[width=0.6\linewidth]{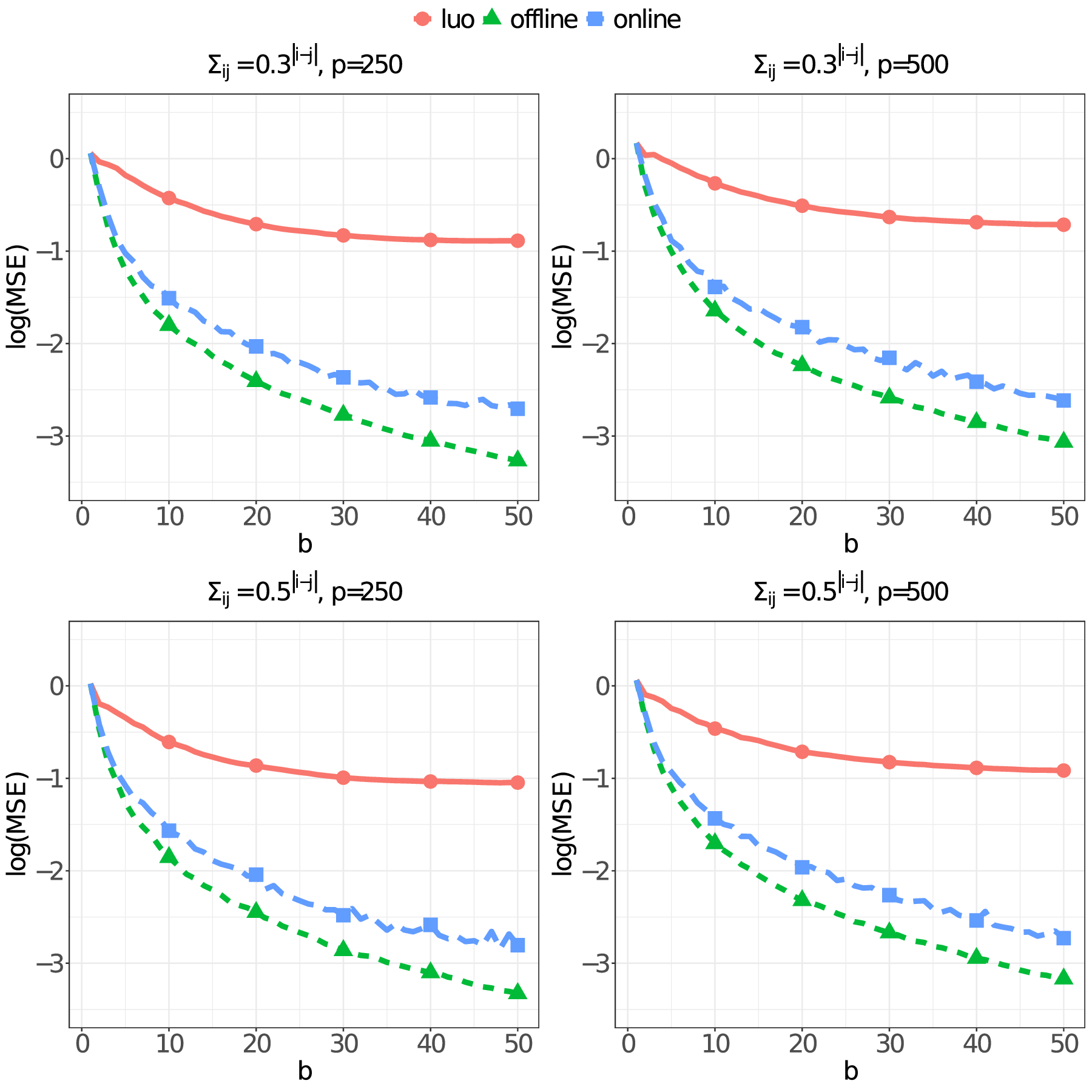}
    \caption{Comparison of $\log(\mathrm{MSE})$ under the logistic model.}
    \label{fig:non_dist_log}
\end{figure}

\subsection{Distributed online estimation}\label{sec:num2}

We next consider the distributed streaming data setting.
For the linear and logistic models, we use similar setups as before. In both models, covariates $\vx_i$ are generated from a multivariate normal distribution with covariance $\vSigma_{ij}=0.5^{|i-j|}$. In the linear model, we use the error distribution $\epsilon_i\sim N(0,1)$. 
We set $B=50$, $n_k^{(b)}=80$, and $p=250$, and the number of sites $K\in\{5,15,25\}$. The results are based on $200$ replications.
We compare the proposed communication-efficient online estimator (\texttt{online}) with the following baselines:
\begin{itemize}
    
    \item[(i)] the centralized offline method (\texttt{cen-off}), which pools all data together and has access to the complete historical sample.

    \item[(ii)] the centralized online method (\texttt{cen-on}), which collects all data for each batch onto one machine and applies the gradient-enhanced online estimator in Section~\ref{sec:online}. 

    \item[(iii)] the local method (\texttt{loc}), which applies the gradient-enhanced online estimator using data from the first machine only.

    \item[(iv)] the one-shot averaging method (\texttt{ave}), which aggregates local estimators from all machines via simple averaging.

\end{itemize}

Figures~\ref{fig:dist_lin} and \ref{fig:dist_log} report $\log(\mathrm{MSE})$ under the linear and logistic models, respectively.
The main findings are as follows.
(i) For a fixed $K$, all methods improve over time, with $\log(\mathrm{MSE})$ decreasing as $b$ increases.
(ii) For a fixed $b$, performance improves as $K$ increases, reflecting the gain from a larger total sample size.
(iii) The \texttt{cen-off} method is the most accurate and serves as a gold-standard benchmark.
The \texttt{cen-on} method outperforms \texttt{online} because the former centralizes all current data.
The proposed  \texttt{online} method substantially improves over \texttt{ave} and \texttt{loc}, showing that the proposed communication-efficient surrogate is considerably more effective than one-shot averaging or single-site learning.

\begin{figure}[H]
    \centering
    \includegraphics[width=\linewidth]{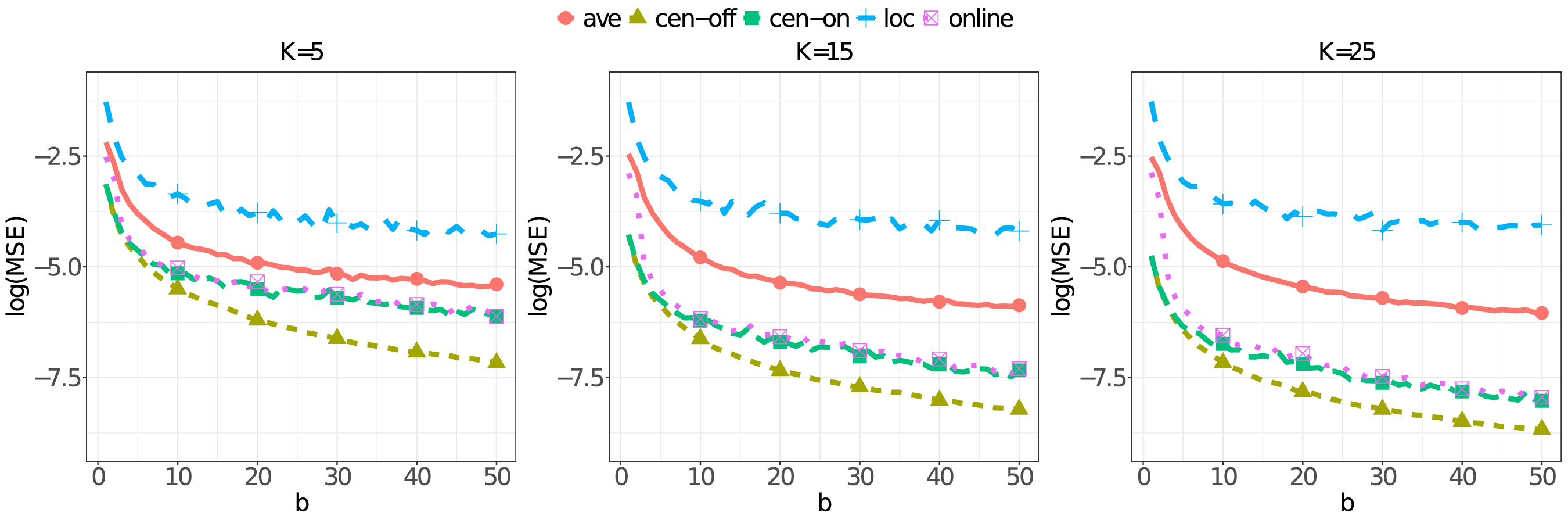}
    \caption{Comparison of $\log(\mbox{MSE})$ in the distributed setting under the linear model.}
    \label{fig:dist_lin}
\end{figure}

\begin{figure}[H]
    \centering
    \includegraphics[width=\linewidth]{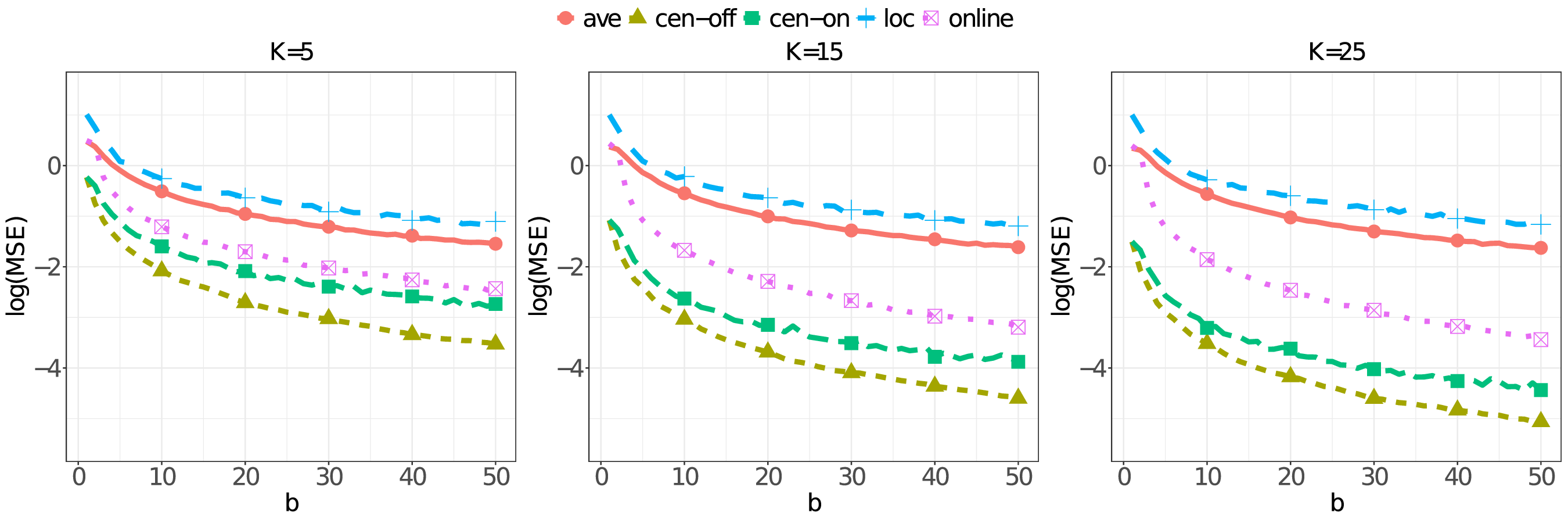}
    \caption{Comparison of $\log(\mbox{MSE})$ in the distributed setting under the logistic model.}
    \label{fig:dist_log}
\end{figure}

\section{Real data application}\label{sec:dataapp}

Online shopping has become an essential component of modern retailing, generating massive streams of user interaction data on e-commerce platforms. 
Accurately modeling and analyzing such large-scale streaming data is of great importance for understanding consumer behavior and improving personalized recommendation systems.

To evaluate the proposed methods, we consider the Taobao ANTA Sneakers User Behavior Dataset (\url{https://www.kaggle.com/w2ds1314/datasets}).
The goal is to predict whether a user will make a purchase based on historical behavior and browsing information.
The data consist of daily user behavior records in May 2025 from seven regions in China ($K=7$), with North China used as the master site.
Because non-purchases are rare, we subsample the data within each region and batch to balance purchase and non-purchase observations at a 1:1 ratio.
The first $B=20$ days are used for training and the remaining 11 days for testing, with 24,754 testing observations.
Figure~\ref{fig:sample} shows the training sample sizes across 7 regions for 20 days.

\begin{figure}[H]
	\centering
	\includegraphics[width=0.8\linewidth]{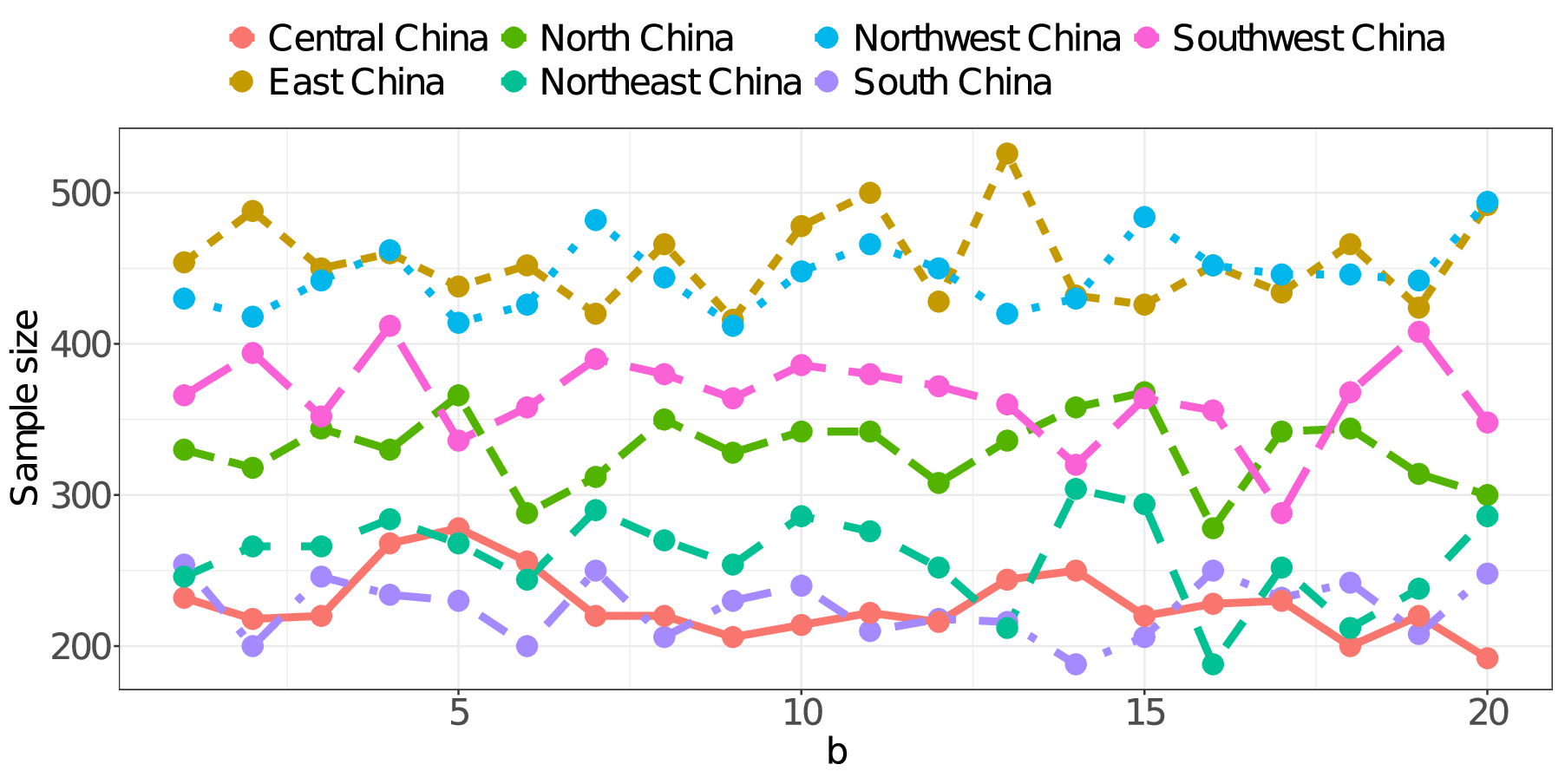}
	\caption{Sample sizes for the streaming training data for different regions.}
	\label{fig:sample}
\end{figure}

The numerical covariates include membership level, number of active days in the past 30 days, historical purchase, price sensitivity, price at browsing, historical sales volume, rating score, number of reviews, number of exposures, number of clicks, number of browsing actions, total browsing duration, number of favorites, number of add-to-cart actions, and number of orders.
All numerical covariates are centered and standardized before model fitting.
After one-hot encoding, the categorical attributes are represented by binary indicators, including Nike, ANTA, and Li-Ning brands; gender; iOS and PC devices; 5G, WiFi, and unknown network types; and activity-page, search, recommendation, external-site, and traffic source channel (such as live-streaming or search).
The total number of covariates is $p=30$.

We compare the five methods in Section~\ref{sec:num2} using the testing log-likelihood; larger values indicate better prediction.
Figure~\ref{fig:loglikelihood} shows that \texttt{cen-off} and \texttt{cen-on} perform best and are nearly indistinguishable.
The proposed distributed \texttt{online} method approaches the centralized benchmarks as more batches arrive and clearly improves over \texttt{loc}.
The one-shot averaging method \texttt{ave} also improves over time, but is less competitive than the proposed surrogate update.

\begin{figure}[H]
	\centering
	\includegraphics[width=0.8\linewidth]{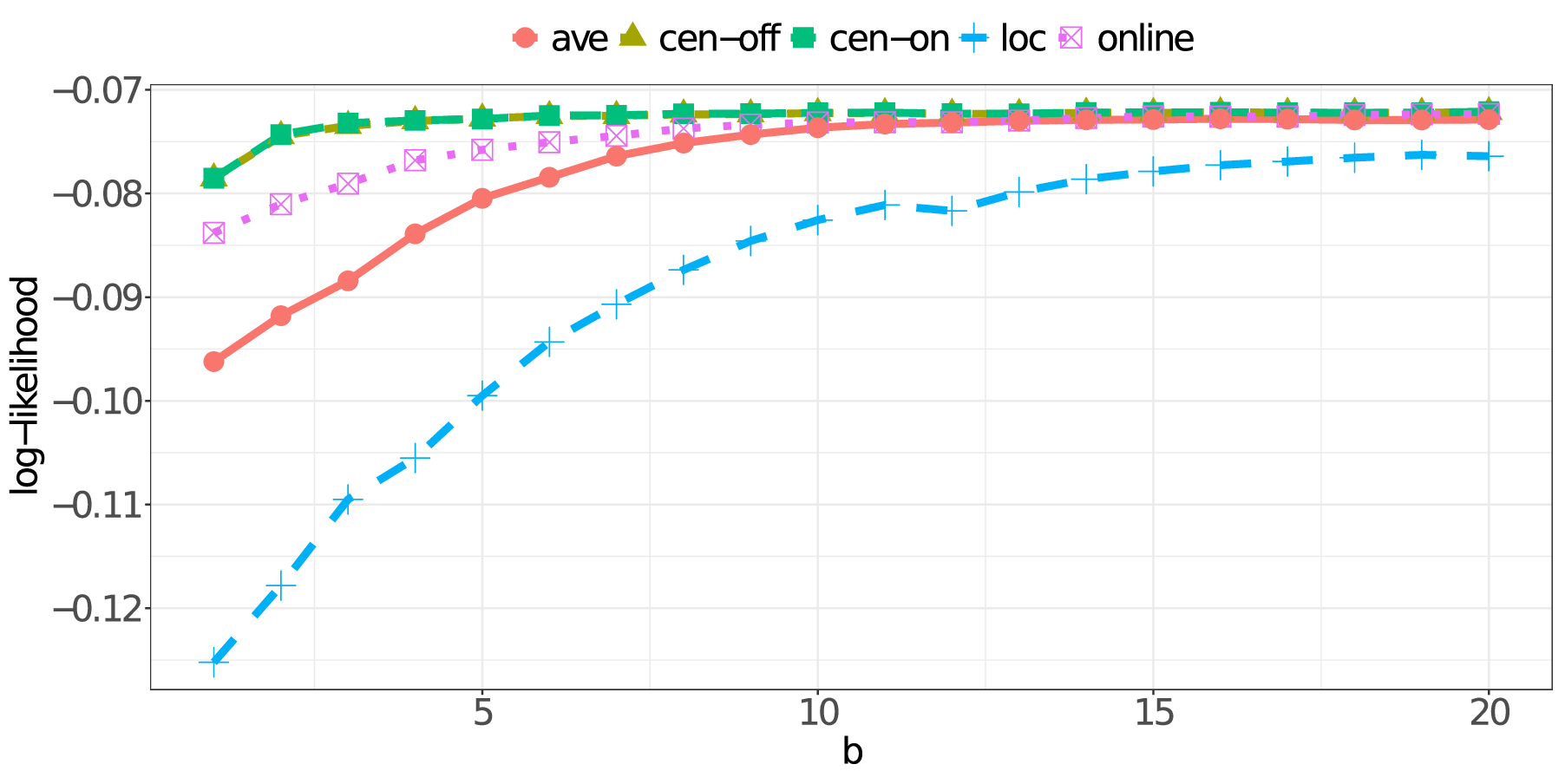}
	\caption{Testing log-likelihood for the Taobao data.}
	\label{fig:loglikelihood}
\end{figure}

We further report parameter MSE in Figure~\ref{fig:MSE}, using full-data estimator \texttt{cen-off} as the truth. The message is similar to that shown by the testing likelihood above.
The local method \texttt{loc} has the largest error, indicating that the master-region data alone are insufficient.
The proposed \texttt{online} estimator is much closer to the centralized estimators than \texttt{loc} and \texttt{ave}.

\begin{figure}[H]
	\centering
	\includegraphics[width=0.8\linewidth]{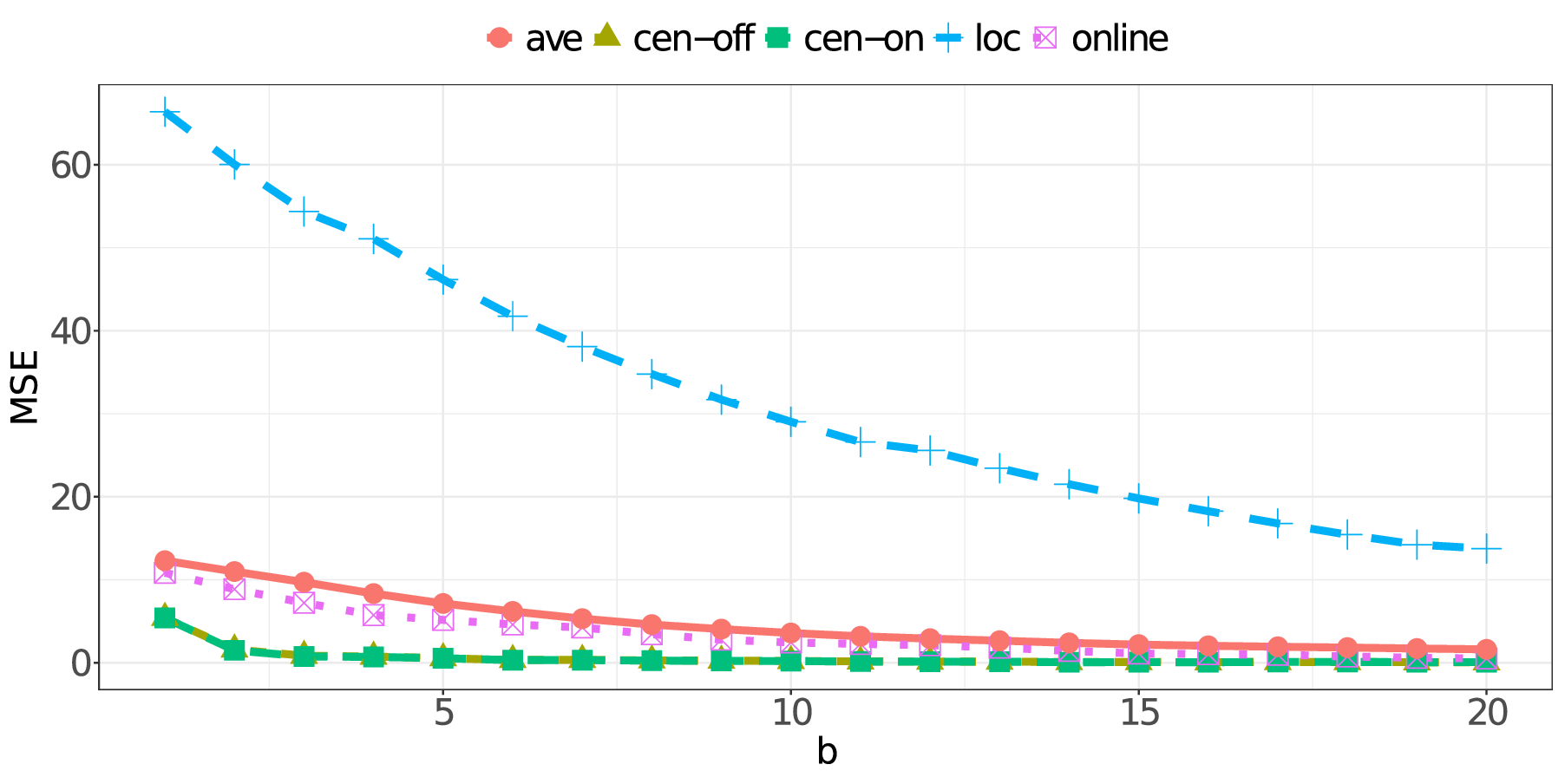}
	\caption{MSE relative to the full-data estimator for the Taobao data.}
	\label{fig:MSE}
\end{figure}

Finally, Table~\ref{tab:estimate} reports point estimates of four  variables selected by \texttt{cen-off} and \texttt{online} with the largest coefficient magnitudes at batch $b=10$ and $b=20$.
All four coefficients are positive: higher membership level, more historical purchases, more orders, and live-streaming source are associated with higher purchase probability.
The coefficient for number of orders is the largest, suggesting that recent ordering behavior is the strongest predictor of purchase conversion.
The estimates from \texttt{online} are close to those from \texttt{cen-off} in both sign and magnitude, indicating that the distributed online estimator recovers similar covariate effects using only streaming summaries.
	
\begin{table}[H]
	\centering
	\small
	\caption{Point estimates of selected important variables in the Taobao data.}
	\vspace{.1in}
	\label{tab:estimate}
	\begin{tabular}{clcccc}
		\hline
		$b$ & Method & Membership level & Historical purchase & Number of orders & Live-streaming \\
		\hline
		10 & \texttt{cen-off} & 0.6565 & 1.0750 & 2.8872 & 1.2155 \\
		    & \texttt{online}  & 0.6970 & 1.1884 & 3.1790 & 1.0262 \\
		20 & \texttt{cen-off} & 0.6690 & 1.1459 & 3.0166 & 1.1187 \\
		    & \texttt{online}  & 0.6292 & 1.0805 & 2.8649 & 0.8219 \\
		\hline
	\end{tabular}
\end{table}

\section{Conclusion and discussions}\label{sec:ext}

In this paper, we have studied online estimation for high-dimensional generalized linear models under both single-site streaming and distributed streaming settings. In the single-site case, we proposed a gradient-enhanced surrogate loss that approximates the cumulative loss via a second-order Taylor expansion, retaining the gradient term that prior renewable estimation approaches discarded. This modification is simple and requires only modest additional storage over the existing approach of \cite{luo2023online}. We established non-asymptotic $\ell_1$- and $\ell_2$-error bounds for the resulting estimator, showing that the proposed estimator achieves the rate $C\sqrt{s^*}\,\lambda^{(b)}$ with a constant $C$ that is uniform in the batch index $b$. This contrasts with the existing theory, in which the analogous factor grows with the number of batches. 

We then extended the framework to a distributed online setting, where $K$ machines each receive new local batches over time and transmit only gradient vectors to a master machine. The distributed surrogate approximates non-master Hessians using the master's own Hessian, so no additional inter-machine Hessian transmission is required. The resulting distributed estimator is communication-efficient: each round requires only gradient vectors from all machines. Theoretical analysis establishes that the same $\ell_1$- and $\ell_2$-convergence rates as the single-site case are achievable. 

Several directions for future work remain. First, the current theory requires the covariate vector to be sub-Gaussian. Extending the analysis to heavier-tailed distributions or dependent covariates, as arising in time-series and panel-data applications, would broaden the scope of the framework. Second, the present method stores a $p \times p$ Hessian matrix on the master, which may be prohibitive when $p$ is very large. Developing sketched or low-rank approximations to the Hessian that preserve the statistical accuracy guarantees is a practically important extension. Third, the current framework assumes that data are homogeneous across sites and batches. Allowing for site-specific heterogeneity in the distribution of covariates or responses, and adapting the surrogate construction accordingly, is an important direction for applications in federated learning. 




\bibliographystyle{natbib}
\bibliography{paper-ref}

 \newpage

\appendix

\setcounter{equation}{0}

\renewcommand{\theequation}{S.\arabic{equation}}

\section*{Supplementary Material}
\section{Technical proofs}

\subsection{Some useful lemmas}

\begin{lem}[Max norm bound for sub-Gaussian vectors]\label{lem:xmax}
    Let $\vx\in\mathbb{R}^p$ be sub-Gaussian with parameter $\sigma_x$ and let $\{\vx_i\}_{i=1}^n$ be i.i.d. copies of $\vx$.
    Then there exists $C>0$ such that
    \begin{equation*}
        \mathbb{P}\Big(\max_{i\le n}\|\vx_i\|_\infty\ge C\sigma_x\sqrt{\log(p\vee n)+t}\Big)\le2\exp\{-C(\log(p\vee n)+t)\}, \quad\forall t>0.
    \end{equation*}
\end{lem}

\noindent\textbf{Proof of Lemma \ref{lem:xmax}.}
Let $\vx=(x_1,\ldots,x_p)\trans$ be sub-Gaussian with parameter $\sigma_x$.
Then, each coordinate $x_\ell$ is sub-Gaussian with the same parameter.
Hence, there exists $C>0$ such that for all $u>0$,
\begin{equation}\label{eq:xmax_tail}
    \mathbb P(|x_{\ell}|\ge u)\le2\exp\{-C\,u^2/\sigma_x^2\}, \quad\forall\ell\in\{1,\ldots,p\}.
\end{equation}
Using \eqref{eq:xmax_tail} and a union bound over $i\in\{1,\ldots,n\}$ and $\ell\in\{1,\ldots,p\}$, we obtain
\begin{equation*}
    \mathbb P\Big(\max_{1\le i\le n}\|\vx_i\|_\infty\ge u\Big)\le 2np\exp\{-C\,u^2/\sigma_x^2\}.
\end{equation*}
Finally, we take $u=C\sigma_x\sqrt{\log(p\vee n)+t}$ to complete the proof.
\hfill $\Box$

\smallskip
\begin{lem}[Restricted strong convexity (RSC)]\label{lem:rsc-info}
    Let $\{\vx_i\}_{i=1}^n$ be i.i.d.
    sub-Gaussian vectors in $\mathbb R^p$ with parameter $\sigma_x$.
    Let $\vSigma:=\mathbb E(\vx\vx\trans)\succeq k_1\vI_p$.
    For a scalar $\tau\asymp 1/\sqrt{\log(p\vee n)}$ and a sufficiently large constant $C_0$, define
    \begin{equation*}
        m_0:=\inf_{|u|\le C_0\|\boldsymbol{\beta}_0\|_2+C_0\tau\sqrt{\log(p\vee n)}}g''(u)>0.
    \end{equation*}
    Then for $n$ satisfying $n\gtrsim s^*\log(p)$, with probability at least $1-\exp\{-Cn\}$,
    \begin{equation*}
        \inf_{\substack{\|\boldsymbol\Delta\|_2=1,\\ 
        \|\boldsymbol\Delta\|_1\le 4\sqrt{s^*},\\
        \|\boldsymbol\Delta'\|_1\le\tau}}
        \vDelta\trans\left(\frac1n\sum_{i=1}^n g''(\vx_i\trans(\vbeta_0+\vDelta'))\vx_i\vx_i\trans\right)\vDelta\ge\kappa_I:=\frac{m_0k_1}{8}.
    \end{equation*}
\end{lem}
\noindent\textbf{Proof of Lemma \ref{lem:rsc-info}.}
Define the truncation events
\begin{align*}
    &\ \calA:=\Big\{|\vx\trans\vbeta_0|\le C_0\|\vbeta_0\|_2,\ \|\vx\|_\infty\le C_0\sqrt{\log(p\vee n)}\Big\},\\
    &\ \calA_i:=\Big\{|\vx_i\trans\vbeta_0|\le C_0\|\vbeta_0\|_2,\ \|\vx_i\|_\infty\le C_0\sqrt{\log(p\vee n)}\Big\}.
\end{align*}
For any $\vDelta'$ with $\|\vDelta'\|_1\le \tau$ and on $\calA_i$,
\begin{equation*}
    |\vx_i\trans\vDelta'|\le\|\vx_i\|_\infty\,\|\vDelta'\|_1\le C_0\tau\sqrt{\log(p\vee n)},
\end{equation*}
and therefore
\begin{equation*}
    |\vx_i\trans(\vbeta_0+\vDelta')|\le|\vx_i\trans\vbeta_0|+|\vx_i\trans\vDelta'|\le C_0\|\vbeta_0\|_2+C_0\tau\sqrt{\log(p\vee n)}.
\end{equation*}
By the definition of $m_0$, this implies
$g''(\vx_i\trans(\vbeta_0+\vDelta'))\ge m_0$ on $\calA_i$.
Hence, for any $\vDelta\in\mathbb R^p$,
\begin{align}
    \vDelta\trans\left(\frac1n\sum_{i=1}^n g''(\vx_i\trans(\vbeta_0+\vDelta'))\,\vx_i \vx_i\trans\right)\vDelta&=\frac1n\sum_{i=1}^n g''(\vx_i\trans(\vbeta_0+\vDelta'))\,(\vx_i\trans\vDelta)^2 \notag\\
    &\ge m_0\cdot\frac1n\sum_{i=1}^n\mathbf1_{\calA_i}(\vx_i\trans\vDelta)^2.\label{eq:rsc_trunc_reduce}
\end{align}
It suffices to establish a uniform lower bound on
$\sum_{i=1}^n\mathbf1_{\calA_i}(\vx_i\trans\vDelta)^2/n$ over the cone.

Fix any $\vDelta$ with $\|\vDelta\|_2=1$.
Let $Z:=\vx\trans\vDelta$.
Since $\vSigma=\mathbb E(\vx\vx\trans)\succeq k_1\vI_p$, we have $\mathbb E(Z^2)\ge k_1\|\vDelta\|_2^2=k_1$.
Besides,
\begin{equation*}
    \mathbb E(Z^2\mathbf 1_{\calA})=\mathbb E(Z^2)-\mathbb E(Z^2\mathbf1_{\calA^c})\ge\mathbb E(Z^2)-\mathbb E(Z^4)^{1/2}\,\mathbb P(\calA^c)^{1/2}.
\end{equation*}
Since $Z$ is sub-Gaussian with $\|Z\|_{\psi_2}\lesssim \sigma_x\|\vDelta\|_2=\sigma_x$, we have $\mathbb E(Z^4)\le C\sigma_x^4$.
Moreover, by sub-Gaussian tail bounds and Lemma~\ref{lem:xmax}, choosing a sufficiently large $C_0$ makes $\mathbb P(\calA^c)$ a sufficiently small constant, so that
$\mathbb E(Z^4)^{1/2}\,\mathbb P(\calA^c)^{1/2}\le k_1/2$.
Thus,
\begin{equation}\label{eq:EZ2_trunc}
    \mathbb E[\mathbf 1_{\calA}(\vx\trans\vDelta)^2]\ge\frac{k_1}{2}>0.
\end{equation}

Notice that $\vu:=\mathbf1_{\calA}\,\vx$ is a sub-Gaussian vector with parameter bounded by $\sigma_x$. 
Applying Lemma 15 of \cite{loh2012high}, for any $s>0$, we have
\begin{equation*}
    \mathbb P\left(\sup_{\substack{\|\boldsymbol\Delta\|_2\le1,\\ 
    \|\boldsymbol\Delta\|_0\le2s}}\left|\frac{1}{n}\sum_{i=1}^n(\vu_i\trans\vDelta)^2-\mathbb E[(\vu\trans\vDelta)^2]\right|\ge t\right)\le2\exp\left\{-Cn\min\left\{\frac{t^2}{\sigma_x^2},\frac{t}{\sigma_x}\right\}+2s\log(p)\right\}.
\end{equation*}
Setting $t=k_1/108$ and $s=nk_1^2/(C\sigma_x^2\,\log(p))$ with a large enough $C>0$, we obtain 
\begin{equation*}
    \sup_{\substack{\|\boldsymbol\Delta\|_2\le1,\\ 
    \|\boldsymbol\Delta\|_0\le2s}}\left|\frac{1}{n}\sum_{i=1}^n(\vu_i\trans\vDelta)^2-\mathbb E[(\vu\trans\vDelta)^2]\right|\le\frac{k_1}{108},
\end{equation*}
with probability at least $1-\exp\{-Cn\}$.
Applying Lemma 13 of \cite{loh2012high},
\begin{equation*}
    \frac{1}{n}\sum_{i=1}^n(\vu_i\trans\vDelta)^2\ge\frac{k_1}{4}\left(\|\vDelta\|_2^2-\frac{\|\vDelta\|_1^2}{s}\right), \quad\forall\,\vDelta\in\mathbb{R}^p.
\end{equation*}  
Under the condition that $n\gtrsim s^*\log(p)$, we have
\begin{equation}\label{eq:Eu2_trunc}
    \frac{1}{n}\sum_{i=1}^n(\vu_i\trans\vDelta)^2\ge\frac{k_1}{8}, \quad\forall\,\|\vDelta\|_2=1,\, \|\vDelta\|_1\le4\sqrt{s^*}.
\end{equation}
Combining the above with \eqref{eq:rsc_trunc_reduce} gives the desired restricted strong convexity bound with constant $\kappa_I=m_0k_1/8$.
\hfill $\Box$

\smallskip
\begin{lem}[Uniform quadratic form upper bound]\label{lem:square}
    Under the same conditions as Lemma~\ref{lem:rsc-info}, for $n\gtrsim s^*\log(p)$, with probability at least $1-\exp\{-Cn\}$,
    \begin{equation*}
        \sup_{\substack{\|\boldsymbol\Delta\|_2=1,\\ \|\boldsymbol\Delta\|_1\le4\sqrt{s^*}}}\frac{1}{n}\sum_{i=1}^n(\vx_i\trans\vDelta)^2\le C.
    \end{equation*}
\end{lem}

\noindent\textbf{Proof of Lemma \ref{lem:square}.}
The proof is similar to that of Lemma \ref{lem:rsc-info}. Applying Lemma 15 of \cite{loh2012high}, for $t=k_1/54$ and $s=nk_1^2/(C\sigma_x^2\log(p))$ with a large enough $C>0$,
\begin{equation*}
    \mathbb P\left(\sup_{\substack{\|\boldsymbol\Delta\|_2\le1,\\ 
    \|\boldsymbol\Delta\|_0\le2s}}\left|\frac{1}{n}\sum_{i=1}^n(\vx_i\trans\vDelta)^2-\mathbb E[(\vx\trans\vDelta)^2]\right|\ge t\right)\le2\exp\{-Cn\},
\end{equation*}
Then applying Lemma 13 of \cite{loh2012high}, we have
\begin{equation*}
    \frac{1}{n}\sum_{i=1}^n(\vx_i\trans\vDelta)^2\le\frac{3}{2}\sigma_x\|\vDelta\|_2^2+\frac{k_1}{2s}\|\vDelta\|_1^2, \quad\forall\,\vDelta\in\mathbb{R}^p.
\end{equation*}  
Using the condition that $n\gtrsim s^*\log(p)$ yields the statement.
\hfill $\Box$

\smallskip
\begin{lem}[Score bound in $\ell_\infty$]\label{lem:grad_inf}
    Let $\varepsilon_i:=y_i-g'(\vx_i\trans\vbeta_0)$.
    Under Assumptions~\ref{ass:design-noise}-\ref{ass:smooth}, there exists $C>0$ such that with probability at least $1-p\exp\{-Cn\}-\exp\{-C(\log(p\vee n)+t)\}$,
    \begin{equation*}
        \left\|\frac{1}{n}\sum_{i=1}^n\varepsilon_i\vx_i\right\|_\infty\le C\sigma_x\sqrt{L_g}\,\sqrt{\frac{\log(p\vee n)+t}{n}}.
    \end{equation*}
\end{lem}

\noindent\textbf{Proof of Lemma \ref{lem:grad_inf}.}
 For a fixed index $j\in\{1,\ldots,p\}$, we begin by establishing an upper bound on the moment generating function of $\sum_{i=1}^n\varepsilon_ix_{ij}/n$.
For any $t\in\mathbb R$,
\begin{align*}
	\log(\mathbb E[\exp\{t\varepsilon_ix_{ij}\}\mid\vx_i])&=\log\{\exp\{-tx_{ij}g'(\vx_i\trans\vbeta_0)\}\,\mathbb E[\exp\{ty_ix_{ij}\}\mid\vx_i]\}\\
	&=g(\vx_i\trans\vbeta_0+tx_{ij})-g(\vx_i\trans\vbeta_0)-tx_{ij}g'(\vx_i\trans\vbeta_0)\\
	&=\frac{t^2x_{ij}^2}{2}g''(\vx_i\trans\vbeta_0+v_{ij}tx_{ij}),
\end{align*}
for some $v_{ij}\in[0,1]$, where the second line is implied by that $g(\vx_i\trans\vbeta_0+t)-g(\vx_i\trans\vbeta_0)$ is the cumulant generating function for the distribution of $y_i|\vx_i$ and the last line is obtained by second-order Taylor expansion.
Then, by Assumption~\ref{ass:smooth}, we have
\begin{equation*}
	\frac{1}{n}\sum_{i=1}^n\log(\mathbb E[\exp\{t\varepsilon_ix_{ij}\}\mid\vx_i])\le\frac{t^2L_g}{2n}\sum_{i=1}^nx_{ij}^2.
\end{equation*}
Under Assumption~\ref{ass:design-noise}, $x_{ij},i\le n$ are i.i.d. sub-Gaussian with parameter $\sigma_x$.
Thus, $x_{ij}^2$ is sub-exponential and $\mathbb E[x_{ij}^2]\le\sigma_x^2$.
Applying Bernstein's inequality, there exists $C>0$ such that
\begin{equation}\label{eq:tail}
	\mathbb P\Big(\frac{1}{n}\sum_{i=1}^nx_{ij}^2\ge C\sigma_x^2\Big)\le\exp\{-Cn\}.
\end{equation}
Now, we define the event $\calE:=\{\max_{j\le p}\sum_{i=1}^nx_{ij}^2/n\le C\sigma_x^2\}$, and we have
\begin{equation*}
	\mathbb P(\calE^c)\le p\exp\{-Cn\}.
\end{equation*}
On $\calE$, we have
\begin{equation*}
	\frac{1}{n}\sum_{i=1}^n\log(\mathbb E[\exp\{t\varepsilon_ix_{ij}\}\mid\vx_i])\le Ct^2L_g\sigma_x^2.
\end{equation*}
Thus, for any $\delta>0$, using Markov's inequality,
\begin{align*}
	&\ \mathbb P\Big(\frac{1}{n}\sum_{i=1}^n\varepsilon_ix_{ij}\ge\delta\mid\calE\Big)\\
	\le&\ \inf_{t>0}\,\exp\{-t\delta\}\,\mathbb E\Big[\exp\Big(\frac{1}{n}\sum_{i=1}^nt\varepsilon_ix_{ij}\Big)\mid\calE\Big]\\
	=&\ \inf_{t>0}\,\exp\{-t\delta\}\,\mathbb E\Big\{\mathbb E\Big[\exp\Big(\frac{1}{n}\sum_{i=1}^nt\varepsilon_ix_{ij}\Big)\mid x_{1j},\ldots,x_{nj},\calE\Big]\mid\calE\Big\}\\
	\le&\ \inf_{t>0}\,\exp\{Ct^2L_g\sigma_x^2/n-t\delta\}\\
	=&\ \exp\{-Cn\delta^2/(L_g\sigma_x^2)\},
\end{align*}
and a similar bound holds for $\mathbb P\Big(\frac{1}{n}\sum_{i=1}^n\varepsilon_ix_{ij}\le -\delta\mid\calE\Big)$. Thus, setting $\delta\asymp\sigma_x\sqrt{L_g(\log(p\vee n)+t)/n}$, we have
\begin{align*}
	\mathbb P\Big(\Big\|\frac{1}{n}\sum_{i=1}^n\varepsilon_i\vx_i\Big\|_\infty\ge\delta\Big)&\le\mathbb P(\calE^c)+\mathbb P\Big(\Big\|\frac{1}{n}\sum_{i=1}^n\varepsilon_i\vx_i\Big\|_\infty\ge\delta\mid\calE\Big)\\
	&\le p\exp\{-Cn\}+\exp\{-C(\log(p\vee n)+t)\}.
\end{align*}
This completes the proof of Lemma~\ref{lem:grad_inf}.
\hfill $\Box$

\smallskip
\begin{lem}[Hessian matrix bound]\label{lem:hessian}
	Under Assumptions~\ref{ass:design-noise}--\ref{ass:smooth}, there exists a constant $C>0$ such that, for $n\gtrsim\log(p)$, with probability at least $1-\exp\{-C(\log(p\vee n)+t)\}$,
	\begin{equation*}
		\left\|\frac{1}{n}\sum_{i=1}^ng''(\vx_i\trans\vbeta_0)\vx_i\vx_i\trans-\mathbb E\{g''(\vx_i\trans\vbeta_0)\vx_i\vx_i\trans\}\right\|_{\max}\le CL_g\sigma_x^2\sqrt{\frac{\log(p\vee n)+t}{n}}.
	\end{equation*}
\end{lem}
	
\noindent\textbf{Proof of Lemma \ref{lem:hessian}.}
Fix $j,k\le p$ and write $Z_{ijk}:=g''(\vx_i\trans\vbeta_0)x_{ij}x_{ik}$.
Under Assumption~\ref{ass:design-noise}, each coordinate $x_{ij}$ is sub-Gaussian with $\|x_{ij}\|_{\psi_2}\le \sigma_x$.
Therefore, the product of two coordinates is sub-exponential and $\|x_{ij}x_{ik}\|_{\psi_1}\le \|x_{ij}\|_{\psi_2}\|x_{ik}\|_{\psi_2}\le \sigma_x^2$.
Hence, $\|Z_{ijk}\|_{\psi_1}\le L_g\|x_{ij}x_{ik}\|_{\psi_1}\le L_g\sigma_x^2$.
Applying Bernstein's inequality, we obtain that for any $\delta>0$,
\begin{equation*}
	\mathbb P\left(\left|\frac{1}{n}\sum_{i=1}^n\{Z_{ijk}-\mathbb E(Z_{ijk})\}\right|>\delta\right)\le2\exp\left\{-Cn\min\left\{\frac{\delta^2}{L_g^2\sigma_x^4}, \frac{\delta}{L_g\sigma_x^2}\right\}\right\}.
\end{equation*}
Taking a union bound over all $p^2$ pairs $(j, k)$ yields
\begin{align*}
	&\,\mathbb P\left(\left\|\frac{1}{n}\sum_{i=1}^ng''(\vx_i\trans\vbeta_0)\vx_i\vx_i\trans-\mathbb E\{g''(\vx_i\trans\vbeta_0)\vx_i\vx_i\trans\}\right\|_{\max}>\delta\right)\\
	\le&\,2\exp\left\{2\log(p)-Cn\min\left\{\frac{\delta^2}{L_g^2\sigma_x^4}, \frac{\delta}{L_g\sigma_x^2}\right\}\right\}.
\end{align*}
Setting $\delta=CL_g\sigma_x^2\sqrt{(\log(p\vee n)+t)/n}$ gives the claimed bound.
\hfill $\Box$

\subsection{Proof of Theorem \ref{thm:online-lasso}}
We prove the bound for a generic batch index $b\le B$ and then take a union bound over $b$.
Let $\calS:=\mathrm{supp}(\vbeta_0)$, $s^*:=|\calS|$,
\begin{align*}
	\breve\calL^{(b)}(\vbeta):=&\ \frac{1}{N^{(b)}}\Big\{\calL^{(b)}(\vbeta)+\vbeta\trans\sum_{j=1}^{b-1}\nabla\calL^{(j)}(\breve\vbeta^{(j)})\\
	&\ \ \ \ \ \ \ \ \ +\frac{1}{2}\sum_{j=1}^{b-1}(\vbeta-\breve\vbeta^{(j)})\trans\nabla^2\calL^{(j)}(\breve\vbeta^{(j)})(\vbeta-\breve\vbeta^{(j)})\Big\},
\end{align*}
$\breve\vbeta^{(b)}=\argmin_{\boldsymbol\beta\in\mathbb{R}^p}\{\breve\calL^{(b)}(\vbeta)+\lambda^{(b)}\|\vbeta\|_1\}$, and $\vDelta^{(b)}:=\breve\vbeta^{(b)}-\vbeta_0$.

\noindent
\textbf{Step 1 (basic inequality and cone condition).}
By optimality of $\breve\vbeta^{(b)}$,
\begin{equation}\label{eq:basic-ineq-online}
    \breve\calL^{(b)}(\vbeta_0+\vDelta^{(b)})-\breve\calL^{(b)}(\vbeta_0)\le\lambda^{(b)}(\|\vbeta_0\|_1-\|\vbeta_0+\vDelta^{(b)}\|_1).
\end{equation}
By convexity of $\breve\calL^{(b)}(\vbeta)$,
\begin{equation}\label{eq:conv-lower-online}
    \breve\calL^{(b)}(\vbeta_0+\vDelta^{(b)})-\breve\calL^{(b)}(\vbeta_0)\ge\lan\nabla\breve\calL^{(b)}(\vbeta_0),\vDelta^{(b)}\ran.
\end{equation}
Assume that
\begin{equation}\label{eq:lambda-score-cond-online}
    \lambda^{(b)}\ge2\|\nabla\breve\calL^{(b)}(\vbeta_0)\|_\infty,
\end{equation}
which will be verified in Step 3.
Then combining \eqref{eq:basic-ineq-online}--\eqref{eq:lambda-score-cond-online} yields
\begin{equation*}
    -\frac{\lambda^{(b)}}{2}\|\vDelta^{(b)}\|_1\le\lambda^{(b)}(\|\vbeta_0\|_1-\|\vbeta_0+\vDelta^{(b)}\|_1).
\end{equation*}
 Using the fact that 
\begin{equation*}
	\|\vbeta_0+\vDelta^{(b)}\|_1\ge\|\vbeta_0\|_1-\|\vDelta_{\calS}^{(b)}\|_1+\|\vDelta_{\calS^c}^{(b)}\|_1,
\end{equation*}
we conclude $\vDelta^{(b)}$ satisfies the cone condition
\begin{equation}\label{eq:online-cone}
	\|\vDelta^{(b)}_{\calS^c}\|_1\le3\|\vDelta^{(b)}_\calS\|_1.
\end{equation}
In particular,
\begin{equation}\label{eq:l1-l2-on-cone}
    \|\vDelta^{(b)}\|_1\le4\|\vDelta^{(b)}_\calS\|_1\le 4\sqrt{s^*}\,\|\vDelta^{(b)}\|_2.
\end{equation}

\noindent
\textbf{Step 2 (restricted strong convexity (RSC)).}
Define the Bregman divergence
\begin{equation*}
    D^{(b)}(\vbeta_0+\vDelta^{(b)}, \vbeta_0):=\lan \nabla\breve\calL^{(b)}(\vbeta_0+\vDelta^{(b)})-\nabla\breve\calL^{(b)}(\vbeta_0), \vDelta^{(b)}\ran.
\end{equation*}
Using the integral form of Taylor expansion and Assumption~\ref{ass:smooth},
\begin{align}\label{eq:breg-int-online}
    D^{(b)}(\vbeta_0+\vDelta^{(b)},\vbeta_0)=&\ \frac{1}{N^{(b)}}\int_0^1\lan\vDelta^{(b)},\nabla^2\calL^{(b)}(\vbeta_0+t\vDelta^{(b)})\vDelta^{(b)}\ran\,dt\notag\\
    &\ +\frac{1}{N^{(b)}}\lan \vDelta^{(b)}, \sum_{j=1}^{b-1}\nabla^2\calL^{(j)}(\breve\vbeta^{(j)})\vDelta^{(b)}\ran.
\end{align}

To obtain a uniform RSC lower bound on the cone, we use the standard truncation argument.
With a slight abuse of notation, define
\begin{equation*}
    \zeta_b:=\min\left\{1, \frac{\tau'}{\|\vDelta^{(b)}\|_2}\right\},
    \quad\vDelta_{\zeta_b}^{(b)}:=\zeta_b\vDelta^{(b)},
\end{equation*}
where $\tau'>0$ is a scalar to be chosen below.
Then $\|\vDelta_{\zeta_b}^{(b)}\|_2\le \tau'$ and  $\vDelta_{\zeta_b}^{(b)}$ also satisfies \eqref{eq:online-cone}.

We apply Lemma~\ref{lem:rsc-info} with $n=n^{(b)}$ and $\tau=4\tau'\sqrt{s^*}$ (induced by \eqref{eq:l1-l2-on-cone}) to the current batch loss $\calL^{(b)}(\vbeta)$.
Under the condition that $n^{(b)}\gtrsim s^*\log(p)$ and $4\tau'\sqrt{s^*\log(p\vee n^{(b)})}=O(1)$ (implied by the scaling condition in \eqref{eq:scale1}), with probability at least $1-\exp\{-Cn^{(b)}\}$,
\begin{equation}\label{eq:rsc-batch-online1}
    \int_0^1\lan\vDelta_{\zeta_b}^{(b)}, \frac{1}{n^{(b)}}\nabla^2\calL^{(b)}(\vbeta_0+t\vDelta_{\zeta_b}^{(b)})\vDelta_{\zeta_b}^{(b)}\ran\,dt\ge\kappa_I\,\|\vDelta_{\zeta_b}^{(b)}\|_2^2.
\end{equation}
 We apply Lemma~\ref{lem:rsc-info} again with $n=n^{(j)}, j\le b-1$  to the historical Hessian $\sum_{j=1}^{b-1}\nabla^2\calL^{(j)}(\breve\vbeta^{(j)})$.
Using the induction hypothesis $\|\breve\vbeta^{(j)}-\vbeta_0\|_2\lesssim\sqrt{s^*\log(p\vee N^{(j)})/N^{(j)}}$, $\|\breve\vbeta^{(j)}-\vbeta_0\|_1\lesssim s^*\sqrt{\log(p\vee N^{(j)})/N^{(j)}}$, 
if $n^{(j)}\gtrsim s^*\log(p)$ and $s^*\sqrt{\log(p\vee n^{(j)})\,\log(p\vee N^{(j)})/N^{(j)}}=O(1)$,  then with probability at least $1-\sum_{j=1}^{b-1}\exp\{-Cn^{(j)}\}$,
\begin{equation}\label{eq:rsc-batch-online2}
	\lan\vDelta_{\zeta_b}^{(b)}, \sum_{j=1}^{b-1}\nabla^2\calL^{(j)}(\breve\vbeta^{(j)})\vDelta_{\zeta_b}^{(b)}\ran\ge\kappa_I(N^{(b)}-n^{(b)})\|\vDelta_{\zeta_b}^{(b)}\|_2^2.
\end{equation}
Finally, combining \eqref{eq:breg-int-online}, \eqref{eq:rsc-batch-online1} and \eqref{eq:rsc-batch-online2} yields the truncated RSC bound
\begin{equation}\label{eq:rsc-trunc-online}
    D^{(b)}(\vbeta_0+\vDelta_{\zeta_b}^{(b)}, \vbeta_0)\ge\kappa_I\,\|\vDelta_{\zeta_b}^{(b)}\|_2^2.
\end{equation}

\noindent
\textbf{Step 3 (score bound to verify \eqref{eq:lambda-score-cond-online}).}
We now bound $\|\nabla\breve\calL^{(b)}(\vbeta_0)\|_\infty$.

(i) For $b=1$, $\breve\calL^{(1)}(\vbeta)$ coincides with the empirical loss on the first batch, so Lemma~\ref{lem:grad_inf} (with $n=N^{(1)}$) yields
\begin{equation*}
    \|\nabla\breve\calL^{(1)}(\vbeta_0)\|_\infty\le C\sqrt{\frac{\log(p\vee N^{(1)})}{N^{(1)}}},
\end{equation*}
with probability at least $1-p\exp\{-CN^{(1)}\}-\exp\{-C\log(p\vee N^{(1)})\}$.

 (ii) For $b\ge2$, we calculate that
\begin{equation*}
	\nabla\breve\calL^{(b)}(\vbeta)=\frac{1}{N^{(b)}}\Big[\nabla\calL^{(b)}(\vbeta)+\sum_{j=1}^{b-1}\{\nabla\calL^{(j)}(\breve\vbeta^{(j)})+\nabla^2\calL^{(j)}(\breve\vbeta^{(j)})(\vbeta-\breve\vbeta^{(j)})\}\Big],
\end{equation*}	
and	thus
\begin{equation*}
	\nabla\breve\calL^{(b)}(\vbeta_0)=-\frac{1}{N^{(b)}}\sum_{i\le N^{(b)}}\varepsilon_i\vx_i+\sum_{j=1}^{b-1}\vR^{(j)},
	\quad\varepsilon_i:=y_i-g'(\vx_i\trans\vbeta_0),
\end{equation*}
where 
\begin{equation}\label{eqn:R}
	\vR^{(j)}:=\frac{1}{N^{(b)}}\Big\{\nabla\calL^{(j)}(\breve\vbeta^{(j)})-\nabla\calL^{(j)}(\vbeta_0)-\nabla^2\calL^{(j)}(\breve\vbeta^{(j)})(\breve\vbeta^{(j)}-\vbeta_0)\Big\}.
\end{equation}

The full-data score is controlled by Lemma~\ref{lem:grad_inf} with $n=N^{(b)}$, which yields 
\begin{equation*}
	\left\|\frac{1}{N^{(b)}}\sum_{i\le N^{(b)}}\varepsilon_i\vx_i\right\|_\infty\le C\sqrt{\frac{\log(p\vee N^{(b)})}{N^{(b)}}},
\end{equation*}
with probability at least $1-p\exp\{-CN^{(b)}\}-\exp\{-C\log(p\vee N^{(b)})\}$.

For the remainder term $\sum_{j=1}^{b-1}\vR^{(j)}$, Taylor's expansion and the boundedness of $g'''$ imply that $\|\vR^{(j)}\|_\infty$ is bounded by a constant multiple of
\begin{equation*}
	L_g\max_{i\in\calD^{(j)}}\|\vx_i\|_\infty\,\frac{1}{N^{(b)}}\sum_{i\in\calD^{(j)}}\{\vx_i\trans(\breve\vbeta^{(j)}-\vbeta_0)\}^2.
\end{equation*}
Summing over $j\le b-1$ and using Lemma~\ref{lem:square}, we get that
 with probability at least $1-\sum_{j=1}^{b-1}\exp\{-Cn^{(j)}\}$.
\begin{equation*}
	\left\|\sum_{j=1}^{b-1}\vR^{(j)}\right\|_\infty\le C\max_{i\le N^{(b)}}\|\vx_i\|_\infty\,\frac{1}{N^{(b)}}\sum_{j=1}^{b-1}n^{(j)}\|\breve\vbeta^{(j)}-\vbeta_0\|_2^2.
\end{equation*}
Using the induction hypothesis that $\|\breve\vbeta^{(j)}-\vbeta_0\|_2\le C\sqrt{s^*\log(p\vee N^{(j)})/N^{(j)}}$ for $j\le b-1$ and applying Lemma~\ref{lem:xmax},
\begin{equation*}
	\max_{i\le N^{(b)}}\|\vx_i\|_\infty\,\frac{1}{N^{(b)}}\sum_{j=1}^{b-1}n^{(j)}\|\breve\vbeta^{(j)}-\vbeta_0\|_2^2\le\frac{Cs^*\log^{3/2}(p\vee N^{(b)})}{N^{(b)}}\sum_{j=1}^{b-1}\frac{n^{(j)}}{N^{(j)}},
\end{equation*}
with probability at least $1-\exp\{-C\log(p\vee N^{(b)})\}$.
Using the fact that $\sum_{j=1}^{b-1}(n^{(j)}/N^{(j)})\le1+\log(N^{(b-1)}/n^{(1)})$ (Lemma 2 of \cite{luo2023online}), we conclude that this remainder is $o(\lambda^{(b)})$ under the scaling condition in \eqref{eq:scale1} and the choice $\lambda^{(b)}=C\sqrt{\log(p\vee N^{(b)})/N^{(b)}}$ for a sufficiently large $C$. 
Thus, we have
$\|\nabla\breve\calL^{(b)}(\vbeta_0)\|_\infty\le\lambda^{(b)}/2$ and this verifies \eqref{eq:lambda-score-cond-online}.

\noindent
\textbf{Step 4 (upper bound on the Bregman divergence).}
 Define $h(t):=\breve\calL^{(b)}(\vbeta_0+t\vDelta^{(b)})$ for $t\in[0,1]$. 
Since $h(t)$ is a convex function, $h'(t)=\lan\nabla\breve\calL^{(b)}(\vbeta_0+t\vDelta^{(b)}), \vDelta^{(b)}\ran$ is nondecreasing in $t$.
This implies
\be\label{eqn:breg}
    D^{(b)}(\vbeta_0+t\vDelta^{(b)},\vbeta_0)=&\ t\{h'(t)-h'(0)\}\nonumber\\
    \le&\ t\{h'(1)-h'(0)\}\nonumber\\
    =&\ tD^{(b)}(\vbeta_0+\vDelta^{(b)},\vbeta_0).
\ee
By the KKT conditions, there exists a subgradient $\vzeta\in\partial\|\breve\vbeta^{(b)}\|_1$ such that
$\nabla\breve\calL^{(b)}(\breve\vbeta^{(b)})+\lambda^{(b)}\vzeta=\mathbf 0$.
Therefore,
\begin{align}
    D^{(b)}(\vbeta_0+\vDelta_{\zeta_b}^{(b)}, \vbeta_0)&\le\zeta_b\,D^{(b)}(\vbeta_0+\vDelta^{(b)}, \vbeta_0)\notag\\
    &=\lan\nabla\breve\calL^{(b)}(\breve\vbeta^{(b)})-\nabla\breve\calL^{(b)}(\vbeta_0), \vDelta_{\zeta_b}^{(b)}\ran\notag\\
    &=-\lambda^{(b)}\lan\vzeta, \vDelta_{\zeta_b}^{(b)}\ran-\lan\nabla\breve\calL^{(b)}(\vbeta_0), \vDelta_{\zeta_b}^{(b)}\ran\notag\\
    &\le\lambda^{(b)}\|\vDelta_{\zeta_b}^{(b)}\|_1+\|\nabla\breve\calL^{(b)}(\vbeta_0)\|_\infty\,\|\vDelta_{\zeta_b}^{(b)}\|_1\notag\\
    &\le\frac{3}{2}\lambda^{(b)}\|\vDelta_{\zeta_b}^{(b)}\|_1.
\label{eq:breg-upper-online}
\end{align}
where the first step used \eqref{eqn:breg} and the last step used \eqref{eq:lambda-score-cond-online}.

\noindent
\textbf{Step 5 (combine bounds).}
Combine the RSC lower bound \eqref{eq:rsc-trunc-online} with the upper bound \eqref{eq:breg-upper-online} and the cone inequality \eqref{eq:l1-l2-on-cone}:
\begin{equation*}
    \kappa_I\,\|\vDelta_{\zeta_b}^{(b)}\|_2^2\le D^{(b)}(\vbeta_0+\vDelta_{\zeta_b}^{(b)},\vbeta_0)\le\frac{3}{2}\lambda^{(b)}\|\vDelta_{\zeta_b}^{(b)}\|_1\le6\lambda^{(b)}\sqrt{s^*}\,\|\vDelta_{\zeta_b}^{(b)}\|_2.
\end{equation*}
Hence
\begin{equation}\label{eq:delta-trunc-l2}
    \|\vDelta_{\zeta_b}^{(b)}\|_2\le C\sqrt{s^*}\,\lambda^{(b)}.
\end{equation}
If $\zeta_b<1$, then $\|\vDelta_{\zeta_b}^{(b)}\|_2=\tau'$ by definition, which contradicts \eqref{eq:delta-trunc-l2} if we choose $\tau'>C\sqrt{s^*}\,\lambda^{(b)}$.
Therefore, $\zeta_b=1$ and $\vDelta_{\zeta_b}^{(b)}=\vDelta^{(b)}$, which implies
\begin{equation*}
    \|\breve\vbeta^{(b)}-\vbeta_0\|_2\le C\sqrt{s^*}\,\lambda^{(b)},
    \quad\|\breve\vbeta^{(b)}-\vbeta_0\|_1\le 4\sqrt{s^*}\,\|\breve\vbeta^{(b)}-\vbeta_0\|_2\le Cs^*\lambda^{(b)}.
\end{equation*}
\hfill $\Box$

\subsection{Proof of Theorem \ref{thm:dist-batch}}

 For $b=1$, we define $\breve\vbeta^{(0)}=\argmin_{\boldsymbol\beta\in\mathbb{R}^p}\{\calL_1^{(1)}(\vbeta)+\lambda^{(0)}\|\vbeta\|_1\}$,
\begin{align}
	\breve\calL^{(1)}(\vbeta):=&\,\frac{1}{N_K^{(1)}}\Big\{\calL_1^{(1)}(\vbeta)+\vbeta\trans\sum_{k=2}^K\nabla\calL_k^{(1)}(\breve\vbeta^{(0)})\notag\\
	&\ \ \ \ \ \ \ +\frac{\sum_{k=2}^Kn_k^{(1)}}{2n_1^{(1)}}(\vbeta-\breve\vbeta^{(0)})\trans\nabla^2\calL_1^{(1)}(\breve\vbeta^{(0)})(\vbeta-\breve\vbeta^{(0)})\Big\},\label{eq:surrogate-loss-b1}
\end{align}
$\breve\vbeta^{(1)}=\argmin_{\boldsymbol\beta\in\mathbb{R}^p}\{\breve\calL^{(1)}(\vbeta)+\lambda^{(1)}\|\vbeta\|_1\}$, and $\vDelta^{(1)}:=\breve\vbeta^{(1)}-\vbeta_0$.
	
For $2\le b\le B$, we define
\begin{align}
	\breve{\calL}^{(b)}(\vbeta):=&\ \frac{1}{N_K^{(b)}}\Big[\calL_1^{(b)}(\vbeta)+\vbeta\trans\Big\{\sum_{k=2}^K\nabla\calL_k^{(b)}(\breve{\vbeta}^{(b-1)})+\sum_{k=1}^K\sum_{j=1}^{b-1}\nabla\calL_k^{(j)}(\breve{\vbeta}^{(j)})\Big\}\notag\\
	&\ \ \ \ \ \ \ \ +\frac{\sum_{k=2}^Kn_k^{(b)}}{2n_1^{(b)}}(\vbeta-\breve{\vbeta}^{(b-1)})\trans\nabla^2\calL_1^{(b)}(\breve{\vbeta}^{(b-1)})(\vbeta-\breve{\vbeta}^{(b-1)})\notag\\
	&\ \ \ \ \ \ \ \ +\sum_{j=1}^{b-1}\frac{\sum_{k=1}^Kn_k^{(j)}}{2n_1^{(j)}}(\vbeta-\breve{\vbeta}^{(j)})\trans\nabla^2\calL_1^{(j)}(\breve{\vbeta}^{(j)})(\vbeta-\breve{\vbeta}^{(j)})\Big],\label{eq:surrogate_loss}
\end{align}
$\breve\vbeta^{(b)}=\argmin_{\boldsymbol\beta\in\mathbb{R}^p}\{\breve\calL^{(b)}(\vbeta)+\lambda^{(b)}\|\vbeta\|_1\}$, and $\vDelta^{(b)}:=\breve\vbeta^{(b)}-\vbeta_0$.

\noindent
\textbf{Step 1 (basic inequality and cone condition).}
By optimality,
\begin{equation*}
	\breve\calL^{(b)}(\vbeta_0+\vDelta^{(b)})-\breve\calL^{(b)}(\vbeta_0)\le\lambda^{(b)}(\|\vbeta_0\|_1-\|\vbeta_0+\vDelta^{(b)}\|_1).
\end{equation*}
By convexity of $\breve\calL^{(b)}(\vbeta)$,
\begin{equation*}
	\breve\calL^{(b)}(\vbeta_0+\vDelta^{(b)})-\breve\calL^{(b)}(\vbeta_0)\ge\lan\nabla\breve\calL^{(b)}(\vbeta_0), \vDelta^{(b)}\ran.
\end{equation*}
Assuming $\lambda^{(b)}\ge2\|\nabla\breve\calL^{(b)}(\vbeta_0)\|_\infty$ (verified in Step 3), the above equations yield
\begin{equation*}
	-\frac{\lambda^{(b)}}{2}\|\vDelta^{(b)}\|_1\le\lambda^{(b)}(\|\vbeta_0\|_1-\|\vbeta_0+\vDelta^{(b)}\|_1).
\end{equation*}
Using the fact that
\begin{equation*}
	\|\vbeta_0+\vDelta^{(b)}\|_1\ge\|\vbeta_0\|_1-\|\vDelta_{\calS}^{(b)}\|_1+\|\vDelta_{\calS^c}^{(b)}\|_1,
\end{equation*} 
we obtain the cone constraint
\begin{equation}\label{eq:dist-cone}
	\|\vDelta_{\calS^c}^{(b)}\|_1\le3\|\vDelta_{\calS}^{(b)}\|_1,
\end{equation}
where $\calS:=\mathrm{supp}(\vbeta_0)$.

\noindent
\textbf{Step 2 (distributed restricted strong convexity).}
We next lower bound the Bregman divergence
$D^{(b)}(\vbeta_0+\vDelta^{(b)}, \vbeta_0):=\langle\nabla\breve\calL^{(b)}(\vbeta_0+\vDelta^{(b)})-\nabla\breve\calL^{(b)}(\vbeta_0), \vDelta^{(b)}\rangle$.

 For $b=1$, using the integral form of Taylor expansion and Assumption~\ref{ass:smooth},
\begin{align*}
	D^{(1)}(\vbeta_0+\vDelta^{(1)}, \vbeta_0)=&\,\frac{1}{N_K^{(1)}}\int_0^1\lan\vDelta^{(1)}, \nabla^2\calL_1^{(1)}(\vbeta_0+t\vDelta^{(1)})\vDelta^{(1)}\ran\,dt\\
	&\,+\frac{1}{N_K^{(1)}}\lan\vDelta^{(1)}, \frac{\sum_{k=2}^Kn_k^{(1)}}{n_1^{(1)}}\nabla^2\calL_1^{(1)}(\breve{\vbeta}^{(0)})\vDelta^{(1)}\ran.
\end{align*}
Then we define
\begin{equation*}
	\zeta_1:=\min\left\{1, \frac{\tau'}{\|\vDelta^{(1)}\|_2}\right\},
	\quad\vDelta_{\zeta_1}^{(1)}:=\zeta_1\vDelta^{(1)},
\end{equation*}
where $\tau'>0$ is a small constant to be chosen.
Then $\|\vDelta_{\zeta_1}^{(1)}\|_2\le\tau'$ and $\vDelta_{\zeta_1}^{(1)}$ also satisfies the cone constraint \eqref{eq:dist-cone} such that $\|\vDelta_{\zeta_1}^{(1)}\|_1\le4\sqrt{s^*}\,\|\vDelta_{\zeta_1}^{(1)}\|_2$.
Applying Lemma~\ref{lem:rsc-info}, under the condition that $n_1^{(1)}\gtrsim s^*\log(p)$ and $4\tau'\sqrt{s^*\log(p\vee n_1^{(1)})}=O(1)$, with probability at least $1-\exp\{-Cn_1^{(1)}\}$,
\begin{equation*}
	\int_0^1\lan\vDelta_{\zeta_1}^{(1)}, \nabla^2\calL_1^{(1)}(\vbeta_0+t\vDelta_{\zeta_1}^{(1)})\vDelta_{\zeta_1}^{(1)}\ran\,dt\ge\kappa_In_1^{(1)}\,\|\vDelta_{\zeta_1}^{(1)}\|_2^2.
\end{equation*}
Using the fact $\|\breve\vbeta^{(0)}-\vbeta_0\|_2\lesssim\sqrt{s^*\log(p\vee n_1^{(1)})/n_1^{(1)}}$, $\|\breve\vbeta^{(0)}-\vbeta_0\|_1\lesssim s^*\sqrt{\log(p\vee n_1^{(1)})/n_1^{(1)}}$, and Lemma~\ref{lem:rsc-info}, if $n_1^{(1)}\gtrsim s^*\log(p)$ and $s^*\log(p\vee n_1^{(1)})/\sqrt{n_1^{(1)}}=O(1)$, then with probability at least $1-\exp\{-Cn_1^{(1)}\}$,
\begin{equation*}
	\lan\vDelta_{\zeta_1}^{(1)}, \nabla^2\calL_1^{(1)}(\breve{\vbeta}^{(0)})\vDelta_{\zeta_1}^{(1)}\ran\ge\kappa_In_1^{(1)}\|\vDelta_{\zeta_1}^{(1)}\|_2^2.
\end{equation*}
Thus, we conclude that $D^{(1)}(\vbeta_0+\vDelta_{\zeta_1}^{(1)}, \vbeta_0)\ge\kappa_I\,\|\vDelta_{\zeta_1}^{(1)}\|_2^2$.

For $2\le b\le B$, using the integral form of Taylor expansion and Assumption~\ref{ass:smooth},
\begin{align*}
	&\ D^{(b)}(\vbeta_0+\vDelta^{(b)}, \vbeta_0)=\frac{1}{N_K^{(b)}}\int_0^1\lan \vDelta^{(b)}, \nabla^2\calL_1^{(b)}(\vbeta_0+t\vDelta^{(b)})\vDelta^{(b)}\ran\,dt\\
	&\ \ \ \ \ +\frac{1}{N_K^{(b)}}\lan\vDelta^{(b)}, \left(\frac{\sum_{k=2}^Kn_k^{(b)}}{n_1^{(b)}}\nabla^2\calL_1^{(b)}(\breve{\vbeta}^{(b-1)})+\sum_{j=1}^{b-1}\frac{\sum_{k=1}^Kn_k^{(j)}}{n_1^{(j)}}\nabla^2\calL_1^{(j)}(\breve{\vbeta}^{(j)})\right)\vDelta^{(b)}\ran.
\end{align*}
Once again, we define
\begin{equation*}
	\zeta_b:=\min\left\{1, \frac{\tau'}{\|\vDelta^{(b)}\|_2}\right\},
	\quad\vDelta_{\zeta_b}^{(b)}:=\zeta_b\vDelta^{(b)}.
\end{equation*}
For the first term in $D^{(b)}(\vbeta_0+\vDelta_{\zeta_b}^{(b)}, \vbeta_0)$, we apply Lemma~\ref{lem:rsc-info} to the master batch at time $b$ with sample size $n_1^{(b)}$.
Under the condition that $n_1^{(b)}\gtrsim s^*\log(p)$ and $4\tau'\sqrt{s^*\log(p\vee n_1^{(b)})}=O(1)$, with probability at least $1-\exp\{-Cn_1^{(b)}\}$,
\begin{equation*}
	\int_0^1\lan\vDelta_{\zeta_b}^{(b)}, \nabla^2\calL_1^{(b)}(\vbeta_0+t\vDelta_{\zeta_b}^{(b)})\vDelta_{\zeta_b}^{(b)}\ran\,dt\ge\kappa_In_1^{(b)}\,\|\vDelta_{\zeta_b}^{(b)}\|_2^2.
\end{equation*}
For the second term in $D^{(b)}(\vbeta_0+\vDelta_{\zeta_b}^{(b)}, \vbeta_0)$, we use Lemma~\ref{lem:rsc-info} again.
Using the fact $\|\breve\vbeta^{(j)}-\vbeta_0\|_2\lesssim\sqrt{s^*\log(p\vee N_K^{(j)})/N_K^{(j)}}$, $\|\breve\vbeta^{(j)}-\vbeta_0\|_1\lesssim s^*\sqrt{\log(p\vee N_K^{(j)})/N_K^{(j)}}$ for $j\le b-1$, if $n_1^{(j)}\gtrsim s^*\log(p)$, $s^*\sqrt{\log(p\vee n_1^{(j)})\,\log(p\vee N_K^{(j)})/N_K^{(j)}}=O(1)$ for $j\le b$, and $s^*\sqrt{\log(p\vee n_1^{(b)})\,\log(p\vee N_K^{(b-1)})/N_K^{(b-1)}}=O(1)$, then with probability at least $1-\sum_{j=1}^b\exp\{-Cn_1^{(j)}\}$,
\begin{align*}
	&\ \lan\vDelta_{\zeta_b}^{(b)}, \left(\frac{\sum_{k=2}^Kn_k^{(b)}}{n_1^{(b)}}\nabla^2\calL_1^{(b)}(\breve{\vbeta}^{(b-1)})+\sum_{j=1}^{b-1}\frac{\sum_{k=1}^Kn_k^{(j)}}{n_1^{(j)}}\nabla^2\calL_1^{(j)}(\breve{\vbeta}^{(j)})\right)\vDelta_{\zeta_b}^{(b)}\ran\\
	\ge&\ \kappa_I(N_K^{(b)}-n_1^{(b)})\,\|\vDelta_{\zeta_b}^{(b)}\|_2^2.
\end{align*}
Thus, we conclude that $D^{(b)}(\vbeta_0+\vDelta_{\zeta_b}^{(b)},\vbeta_0)\ge \kappa_I\,\|\vDelta_{\zeta_b}^{(b)}\|_2^2$.

\noindent
\textbf{Step 3 (score bound at $\vbeta_0$).}
We now verify the score condition
\begin{equation}\label{eq:dist-score}
	\|\nabla\breve\calL^{(b)}(\vbeta_0)\|_\infty\le \frac{\lambda^{(b)}}{2}.
\end{equation}

For $b=1$, by the definition of $\breve\calL^{(1)}(\vbeta)$ in \eqref{eq:surrogate-loss-b1}, we calculate that
\begin{equation*}
	\nabla\breve\calL^{(1)}(\vbeta)=\frac{1}{N_K^{(1)}}\Big\{\nabla\calL_1^{(1)}(\vbeta)+\sum_{k=2}^K\nabla\calL_k^{(1)}(\breve\vbeta^{(0)})+\frac{\sum_{k=2}^Kn_k^{(1)}}{n_1^{(1)}}\nabla^2\calL_1^{(1)}(\breve\vbeta^{(0)})(\vbeta-\breve\vbeta^{(0)})\Big\}.
\end{equation*}
For $2\le b\le B$, by the definition of $\breve\calL^{(b)}(\vbeta)$ in \eqref{eq:surrogate_loss}, we calculate that
\begin{align*}
	&\ \nabla\breve\calL^{(b)}(\vbeta)=\frac{1}{N_K^{(b)}}\Big[\nabla\calL_1^{(b)}(\vbeta)+\sum_{k=2}^K\nabla\calL_k^{(b)}(\breve\vbeta^{(b-1)})+\sum_{k=1}^K\sum_{j=1}^{b-1}\nabla\calL_k^{(j)}(\breve\vbeta^{(j)})\\
	&\ \ +\frac{\sum_{k=2}^Kn_k^{(b)}}{n_1^{(b)}}\nabla^2\calL_1^{(b)}(\breve\vbeta^{(b-1)})(\vbeta-\breve\vbeta^{(b-1)})+\sum_{j=1}^{b-1}\frac{\sum_{k=1}^Kn_k^{(j)}}{n_1^{(j)}}\nabla^2\calL_1^{(j)}(\breve\vbeta^{(j)})(\vbeta-\breve\vbeta^{(j)})\Big].
\end{align*}

\noindent
\textbf{Step 3.1 (decompose into full-data empirical score and approximation remainder).}
 The full-data score at time $b$ is
\begin{equation*}
	\frac{1}{N_K^{(b)}}\sum_{k=1}^K\sum_{j=1}^{b}\nabla\calL_k^{(j)}(\vbeta_0)=-\frac{1}{N_K^{(b)}}\sum_{i\le N_K^{(b)}}\varepsilon_i\vx_i,
\end{equation*}
where $\varepsilon_i:=y_i-g'(\vx_i\trans\vbeta_0)$.
Then we can rewrite
\begin{equation}\label{eq:score-decomp-b1}
	\nabla\breve\calL^{(1)}(\vbeta_0)=\frac{1}{N_K^{(1)}}\sum_{k=1}^K\nabla\calL_k^{(1)}(\vbeta_0)+\vR^{(1)},
\end{equation}
where the remainder is expressed as
\begin{align*}
	N_K^{(1)}\,\vR^{(1)}=-\sum_{k=2}^K\{\nabla\calL_k^{(1)}(\vbeta_0)-\nabla\calL_k^{(1)}(\breve\vbeta^{(0)})\}+\frac{\sum_{k=2}^Kn_k^{(1)}}{n_1^{(1)}}\nabla^2\calL_1^{(1)}(\breve{\vbeta}^{(0)})(\vbeta_0-\breve{\vbeta}^{(0)}).
\end{align*}
For $2\le b\le B$,
\begin{equation}\label{eq:score-decomp}
	\nabla\breve\calL^{(b)}(\vbeta_0)=\frac{1}{N_K^{(b)}}\sum_{k=1}^K\sum_{j=1}^{b}\nabla\calL_k^{(j)}(\vbeta_0)+\vR^{(b)},
\end{equation}
where the remainder is expressed as
\begin{align*}
	&\,N_K^{(b)}\,\vR^{(b)}\\
	=&\,-\sum_{k=2}^K\{\nabla\calL_k^{(b)}(\vbeta_0)-\nabla\calL_k^{(b)}(\breve\vbeta^{(b-1)})\}-\sum_{k=1}^K\sum_{j=1}^{b-1}\{\nabla\calL_k^{(j)}(\vbeta_0)-\nabla\calL_k^{(j)}(\breve\vbeta^{(j)})\}\\
	&\,+\frac{\sum_{k=2}^Kn_k^{(b)}}{n_1^{(b)}}\nabla^2\calL_1^{(b)}(\breve{\vbeta}^{(b-1)})(\vbeta_0-\breve{\vbeta}^{(b-1)})+\sum_{j=1}^{b-1}\frac{\sum_{k=1}^Kn_k^{(j)}}{n_1^{(j)}}\nabla^2\calL_1^{(j)}(\breve{\vbeta}^{(j)})(\vbeta_0-\breve{\vbeta}^{(j)}).
\end{align*}

The full-data empirical score $\sum_{k=1}^K\sum_{j=1}^{b}\nabla\calL_k^{(j)}(\vbeta_0)/N_K^{(b)}$ for $b\le B$ is controlled by Lemma~\ref{lem:grad_inf} with $n=N_K^{(b)}$: with probability at least $1-p\exp\{-CN_K^{(b)}\}-\exp\{-C\log(p\vee N_K^{(b)})\}$,
\begin{equation}\label{eq:score-emp}
	\left\|\frac{1}{N_K^{(b)}}\sum_{i\le N_K^{(b)}}\varepsilon_i\vx_i\right\|_\infty\le C\sqrt{\frac{\log(p\vee N_K^{(b)})}{N_K^{(b)}}}.
\end{equation}
It remains to bound $\|\vR^{(b)}\|_\infty$.

\noindent
\textbf{Step 3.2 (bound on the renewable approximation remainder).}
 (i) For batch $b=1$ and site $2\le k\le K$, a second-order Taylor expansion of the local score gives
\begin{equation}\label{eq:taylor-score1-b1}
	\nabla\calL_k^{(1)}(\vbeta_0)=\nabla\calL_k^{(1)}(\breve\vbeta^{(0)})+\nabla^2\calL_k^{(1)}(\breve\vbeta^{(0)})(\vbeta_0-\breve\vbeta^{(0)})+\vT_k^{(1)},
\end{equation}
where the remainder term satisfies (by bounded $g'''$)
\begin{equation}\label{eq:taylor-rem-bound1-b1}
	\|\vT_k^{(1)}\|_\infty\le CL_g\max_{i\in\mathcal D_k^{(1)}}\|\vx_i\|_\infty\sum_{i\in\mathcal D_k^{(1)}}\{\vx_i\trans(\breve\vbeta^{(0)}-\vbeta_0)\}^2.
\end{equation}
Summing \eqref{eq:taylor-score1-b1} over $2\le k\le K$, we have
\begin{equation}\label{eq:Rb-bound-final-b1}
	N_K^{(1)}\,\vR^{(1)}=\left\{\sum_{k=2}^K\nabla^2\calL_k^{(1)}(\breve\vbeta^{(0)})-\frac{\sum_{k=2}^Kn_k^{(1)}}{n_1^{(1)}}\nabla^2\calL_1^{(1)}(\breve{\vbeta}^{(0)})\right\}(\breve{\vbeta}^{(0)}-\vbeta_0)-\sum_{k=2}^K\vT_k^{(1)}.
\end{equation}

For the term $\sum_{k=2}^K\vT_k^{(1)}$ in \eqref{eq:Rb-bound-final-b1}, applying Lemma~\ref{lem:xmax} and Lemma~\ref{lem:square} and the fact that $\|\breve\vbeta^{(0)}-\vbeta_0\|_2\lesssim\sqrt{s^*\log(p\vee n_1^{(1)})/n_1^{(1)}}$, $\|\breve\vbeta^{(0)}-\vbeta_0\|_1\lesssim s^*\sqrt{\log(p\vee n_1^{(1)})/n_1^{(1)}}$, we obtain that 
with probability at least $1-\exp\{-C\sum_{k=2}^K n_k^{(1)}\}$.
Plugging the above inequalities into \eqref{eq:taylor-rem-bound1-b1} yields
\begin{equation*}
	\frac{1}{N_K^{(1)}}\left\|\sum_{k=2}^K\vT_k^{(1)}\right\|_\infty\le\frac{Cs^*(N_K^{(1)}-n_1^{(1)})\log(p\vee n_1^{(1)})\sqrt{\log(p\vee N_K^{(1)})}}{n_1^{(1)}N_K^{(1)}}.
\end{equation*}
This remainder is $o(\lambda^{(1)})$ under the scaling condition that $s^*(N_K^{(1)}-n_1^{(1)})\log(p\vee n_1^{(1)})/(n_1^{(1)}\sqrt{N_K^{(1)}})=o(1)$ and the choice $\lambda^{(1)}=C\sqrt{\log(p\vee N_K^{(1)})/N_K^{(1)}}$ for a sufficiently large $C$.

For the term $\{\sum_{k=2}^K\nabla^2\calL_k^{(1)}(\breve\vbeta^{(0)})-\sum_{k=2}^K(n_k^{(1)}/n_1^{(1)})\nabla^2\calL_1^{(1)}(\breve{\vbeta}^{(0)})\}(\breve{\vbeta}^{(0)}-\vbeta_0)$ in \eqref{eq:Rb-bound-final-b1}, we observe that:
\begin{equation*}
	\left\{\frac{1}{\sum_{k=2}^Kn_k^{(1)}}\sum_{k=2}^K\nabla^2\calL_k^{(1)}(\breve\vbeta^{(0)})-\frac{1}{n_1^{(1)}}\nabla^2\calL_1^{(1)}(\breve{\vbeta}^{(0)})\right\}(\breve{\vbeta}^{(0)}-\vbeta_0)\le I_1^{(1)}+I_2^{(1)}+I_3^{(1)}+I_4^{(1)},
\end{equation*}
where $I_\ell^{(1)}, \ell=1,\dots,4$ will be defined and bounded below by applying Lemmas~\ref{lem:xmax}, \ref{lem:square} and \ref{lem:hessian}.
For $I_1^{(1)}$, with probability at least $1-\exp\{-C\log(p\vee\sum_{k=2}^Kn_k^{(1)})\}-\exp\{-C\sum_{k=2}^Kn_k^{(1)}\}$,
\begin{align*}
	I_1^{(1)}:=&\,\left\|\frac{1}{\sum_{k=2}^Kn_k^{(1)}}\sum_{k=2}^K\sum_{i\in\calD_k^{(1)}}\{g''(\vx_i\trans\breve\vbeta^{(0)})-g''(\vx_i\trans\vbeta_0)\}\vx_i\vx_i\trans(\breve{\vbeta}^{(0)}-\vbeta_0)\right\|_\infty\\
	\le&\,CL_g\max_{i\in\cup_{k=2}^K\calD_k^{(1)}}\|\vx_i\|_\infty\,\frac{1}{\sum_{k=2}^Kn_k^{(1)}}\sum_{k=2}^K\sum_{i\in\calD_k^{(1)}}\{\vx_i\trans(\breve\vbeta^{(0)}-\vbeta_0)\}^2\\
	\le&\,\frac{Cs^*\log(p\vee n_1^{(1)})\sqrt{\log(p\vee \sum_{k=2}^Kn_k^{(1)})}}{n_1^{(1)}}.
\end{align*}
For $I_2^{(1)}$, by Lemma \ref{lem:hessian}, with probability at least $1-\exp\{-C\log(p\vee\sum_{k=2}^Kn_k^{(1)})\}$,
\begin{align*}
	I_2^{(1)}:=&\,\left\|\left[\frac{1}{\sum_{k=2}^Kn_k^{(1)}}\sum_{k=2}^K\sum_{i\in\calD_k^{(1)}}g''(\vx_i\trans\vbeta_0)\vx_i\vx_i\trans-\mathbb E\{g''(\vx_i\trans\vbeta_0)\vx_i\vx_i\trans\}\right](\breve\vbeta^{(0)}-\vbeta_0)\right\|_\infty\\
	\le&\,\left\|\frac{1}{\sum_{k=2}^Kn_k^{(1)}}\sum_{k=2}^K\sum_{i\in\calD_k^{(1)}}g''(\vx_i\trans\vbeta_0)\vx_i\vx_i\trans-\mathbb E\{g''(\vx_i\trans\vbeta_0)\vx_i\vx_i\trans\}\right\|_{\max}\|\breve\vbeta^{(0)}-\vbeta_0\|_1\\
	\le&\,Cs^*\sqrt{\frac{\log(p\vee n_1^{(1)})\,\log(p\vee\sum_{k=2}^Kn_k^{(1)})}{n_1^{(1)}\sum_{k=2}^Kn_k^{(1)}}}.
\end{align*}
For $I_3^{(1)}$, with probability at least $1-\exp\{-C\log(p\vee n_1^{(1)})\}$,
\begin{align*}
	I_3^{(1)}:=&\,\left\|\left[\frac{1}{n_1^{(1)}}\sum_{i\in\calD_1^{(1)}}g''(\vx_i\trans\vbeta_0)\vx_i\vx_i\trans-\mathbb E\{g''(\vx_i\trans\vbeta_0)\vx_i\vx_i\trans\}\right](\breve\vbeta^{(0)}-\vbeta_0)\right\|_\infty\\
	\le&\,\left\|\frac{1}{n_1^{(1)}}\sum_{i\in\calD_1^{(1)}}g''(\vx_i\trans\vbeta_0)\vx_i\vx_i\trans-\mathbb E\{g''(\vx_i\trans\vbeta_0)\vx_i\vx_i\trans\}\right\|_{\max}\|\breve\vbeta^{(0)}-\vbeta_0\|_1\\
	\le&\,\frac{Cs^*\log(p\vee n_1^{(1)})}{n_1^{(1)}}.
\end{align*}
For $I_4^{(1)}$, with probability at least $1-\exp\{-C\log(p\vee n_1^{(1)})\}-\exp\{-Cn_1^{(1)}\}$,
\begin{align*}
	I_4^{(1)}:=&\,\left\|\frac{1}{n_1^{(1)}}\sum_{i\in\calD_1^{(1)}}\{g''(\vx_i\trans\vbeta_0)-g''(\vx_i\trans\breve\vbeta^{(0)})\}\vx_i\vx_i\trans(\breve\vbeta^{(0)}-\vbeta_0)\right\|_\infty\\
	\le&\,CL_g\max_{i\in\calD_1^{(1)}}\|\vx_i\|_\infty\,\frac{1}{n_1^{(1)}}\sum_{i\in\calD_1^{(1)}}\{\vx_i\trans(\breve\vbeta^{(0)}-\vbeta_0)\}^2\\
	\le&\,\frac{Cs^*\log^{3/2}(p\vee n_1^{(1)})}{n_1^{(1)}}.
\end{align*}
Thus, we have
\begin{align*}
	&\,\frac{1}{N_K^{(1)}}\left\|\left\{\sum_{k=2}^K\nabla^2\calL_k^{(1)}(\breve\vbeta^{(0)})-\frac{\sum_{k=2}^Kn_k^{(1)}}{n_1^{(1)}}\nabla^2\calL_1^{(1)}(\breve{\vbeta}^{(0)})\right\}(\breve{\vbeta}^{(0)}-\vbeta_0)\right\|_\infty\\
	\le&\,\frac{Cs^*(N_K^{(1)}-n_1^{(1)})\log(p\vee n_1^{(1)})\sqrt{\log(p\vee N_K^{(1)})}}{n_1^{(1)}N_K^{(1)}},
\end{align*}
which is also $o(\lambda^{(1)})$.
Thus, we conclude that $\|\vR^{(1)}\|_\infty=o(\lambda^{(1)})$.

(ii) For batch $2\le b\le B$, site $k\le K$ and batch $j\le b-1$, a second-order Taylor expansion of the local score gives
\begin{equation}\label{eq:taylor-score1}
	\nabla\calL_k^{(j)}(\vbeta_0)=\nabla\calL_k^{(j)}(\breve\vbeta^{(j)})+\nabla^2\calL_k^{(j)}(\breve\vbeta^{(j)})(\vbeta_0-\breve\vbeta^{(j)})+\vT_{k}^{(j)},
\end{equation}
where the remainder term satisfies (by bounded $g'''$)
\begin{equation}\label{eq:taylor-rem-bound1}
	\|\vT_k^{(j)}\|_\infty\le CL_g\,\max_{i\in\mathcal D_k^{(j)}}\|\vx_i\|_\infty\sum_{i\in\mathcal D_k^{(j)}}\{\vx_i\trans(\breve\vbeta^{(j)}-\vbeta_0)\}^2.
\end{equation}
Similarly, for site $2\le k\le K$ and batch $b$,
\begin{equation}\label{eq:taylor-score2}
	\nabla\calL_k^{(b)}(\vbeta_0)=\nabla\calL_k^{(b)}(\breve\vbeta^{(b-1)})+\nabla^2\calL_k^{(b)}(\breve\vbeta^{(b-1)})(\vbeta_0-\breve\vbeta^{(b-1)})+\vT_{k}^{(b)},
\end{equation}
where the remainder term satisfies
\begin{equation}\label{eq:taylor-rem-bound2}
	\|\vT_k^{(b)}\|_\infty\le CL_g\max_{i\in\mathcal D_k^{(b)}}\|\vx_i\|_\infty\sum_{i\in\mathcal D_k^{(b)}}\{\vx_i\trans(\breve\vbeta^{(b-1)}-\vbeta_0)\}^2.
\end{equation}
Summing \eqref{eq:taylor-score1} over $k\le K,j\le b-1$ and \eqref{eq:taylor-score2} over $k\in\{2,\ldots,K\}$, we have
\begin{align}
	N_K^{(b)}\,\vR^{(b)}=&\ \left\{\sum_{k=2}^K\nabla^2\calL_k^{(b)}(\breve\vbeta^{(b-1)})-\frac{\sum_{k=2}^Kn_k^{(b)}}{n_1^{(b)}}\nabla^2\calL_1^{(b)}(\breve{\vbeta}^{(b-1)})\right\}(\breve{\vbeta}^{(b-1)}-\vbeta_0)\notag\\
	&{}+\sum_{j=1}^{b-1}\left\{\sum_{k=1}^K\nabla^2\calL_k^{(j)}(\breve\vbeta^{(j)})-\frac{\sum_{k=1}^Kn_k^{(j)}}{n_1^{(j)}}\nabla^2\calL_1^{(j)}(\breve{\vbeta}^{(j)})\right\}(\breve{\vbeta}^{(j)}-\vbeta_0)\notag\\
	&{}-\sum_{k=2}^K\vT_k^{(b)}-\sum_{k=1}^K\sum_{j=1}^{b-1}\vT_k^{(j)}\label{eq:Rb-bound-final}.
\end{align}

 Applying Lemma~\ref{lem:xmax}, we have $\max_{i\le N_K^{(b)}}\|\vx_i\|_\infty\le C\sigma_x\sqrt{\log(p\vee N_K^{(b)})}$ with probability at least $1-\exp\{-C\log(p\vee N_K^{(b)})\}$.
Using Lemma~\ref{lem:square} and the induction hypothesis that $\|\breve\vbeta^{(j)}-\vbeta_0\|_2\lesssim \sqrt{s^*\log(p\vee N_K^{(j)})/N_K^{(j)}}$, $\|\breve\vbeta^{(j)}-\vbeta_0\|_1\lesssim s^*\sqrt{\log(p\vee N_K^{(j)})/N_K^{(j)}}$ for $j\le b-1$, we obtain uniformly for $k\le K$, $j\le b$ and $l\le b-1$,
\begin{equation*}
	\sum_{i\in\calD_k^{(j)}}\{\vx_i\trans(\breve\vbeta^{(l)}-\vbeta_0)\}^2\le Cn_k^{(j)}\,\|\breve\vbeta^{(l)}-\vbeta_0\|_2^2\le\frac{Cs^*n_k^{(j)}\log(p\vee N_K^{(l)})}{N_K^{(l)}},
\end{equation*}
with probability at least $1-\sum_{j=1}^{b-1}\exp\{-C\sum_{k=1}^Kn_k^{(j)}\}-\exp\{-C\sum_{k=2}^Kn_k^{(b)}\}$,
\begin{align*}
	&\,\frac{1}{N_K^{(b)}}\left\|\sum_{k=2}^K\vT_k^{(b)}+\sum_{k=1}^K\sum_{j=1}^{b-1}\vT_k^{(j)}\right\|_\infty\\
	\le&\,C L_g\max_{i\le N_K^{(b)}}\|\vx_i\|_\infty\left\{\frac{s^*(\sum_{k=2}^Kn_k^{(b)})\log(p\vee N_K^{(b-1)})}{N_K^{(b-1)}N_K^{(b)}}+\frac{s^*}{N_K^{(b)}}\sum_{k=1}^K\sum_{j=1}^{b-1}\frac{n_k^{(j)}\log(p\vee N_K^{(j)})}{N_K^{(j)}}\right\}\\
	\le&\,\frac{CL_gs^*\log^{3/2}(p\vee N_K^{(b)})}{N_K^{(b)}}\left\{\frac{N_K^{(b)}-N_K^{(b-1)}-n_1^{(b)}}{N_K^{(b-1)}}+1+\log\Big(\frac{N_K^{(b-1)}}{N_K^{(1)}}\Big)\right\},
\end{align*}
where the last inequality is implied by Lemma 2 of \cite{luo2023online}:
\begin{equation*}
	\sum_{k=1}^K\sum_{j=1}^{b-1}\frac{n_k^{(j)}}{N_K^{(j)}}=\sum_{j=1}^{b-1}\frac{\sum_{k=1}^Kn_k^{(j)}}{N_K^{(j)}}\le1+\log\Big(\frac{N_K^{(b-1)}}{N_K^{(1)}}\Big).
\end{equation*}
This remainder is $o(\lambda^{(b)})$ under the scaling condition in \eqref{eq:scale2} and the choice $\lambda^{(b)}=C\sqrt{\log(p\vee N_K^{(b)})/N_K^{(b)}}$ for a sufficiently large $C$.

To control the first line of \eqref{eq:Rb-bound-final}, we will apply Lemmas~\ref{lem:xmax}, \ref{lem:square} and \ref{lem:hessian} below.
Notice that
\begin{align*}
	&\,\left\|\left\{\frac{1}{\sum_{k=2}^Kn_k^{(b)}}\sum_{k=2}^K\nabla^2\calL_k^{(b)}(\breve\vbeta^{(b-1)})-\frac{1}{n_1^{(b)}}\nabla^2\calL_1^{(b)}(\breve\vbeta^{(b-1)})\right\}(\breve\vbeta^{(b-1)}-\vbeta_0)\right\|_\infty\\
	\le&\,I_1^{(b)}+I_2^{(b)}+I_3^{(b)}+I_4^{(b)},
\end{align*}
where $I_\ell^{(b)}, \ell=1,\ldots,4$ will be defined and bounded below.
For $I_1^{(b)}$, with probability at least $1-\exp\{-C\log(p\vee\sum_{k=2}^Kn_k^{(b)})\}-\exp\{-C\sum_{k=2}^Kn_k^{(b)}\}$,
\begin{align*}
	I_1^{(b)}:=&\,\left\|\frac{1}{\sum_{k=2}^Kn_k^{(b)}}\sum_{k=2}^K\sum_{i\in\calD_k^{(b)}}\{g''(\vx_i\trans\breve\vbeta^{(b-1)})-g''(\vx_i\trans\vbeta_0)\}\vx_i\vx_i\trans(\breve\vbeta^{(b-1)}-\vbeta_0)\right\|_\infty\\
	\le&\,CL_g\max_{i\in\cup_{k=2}^K\calD_k^{(b)}}\|\vx_i\|_\infty\,\frac{1}{\sum_{k=2}^Kn_k^{(b)}}\sum_{k=2}^K\sum_{i\in\calD_k^{(b)}}\{\vx_i\trans(\breve\vbeta^{(b-1)}-\vbeta_0)\}^2\\
	\le&\,\frac{Cs^*\log(p\vee N_K^{(b-1)})\sqrt{\log(p\vee \sum_{k=2}^Kn_k^{(b)})}}{N_K^{(b-1)}}.
\end{align*}
For $I_2^{(b)}$, with probability at least $1-\exp\{-C\log(p\vee\sum_{k=2}^Kn_k^{(b)})\}$,
\begin{align*}
	I_2^{(b)}:=&\,\left\|\left[\frac{1}{\sum_{k=2}^Kn_k^{(b)}}\sum_{k=2}^K\sum_{i\in\calD_k^{(b)}}g''(\vx_i\trans\vbeta_0)\vx_i\vx_i\trans-\mathbb E\{g''(\vx_i\trans\vbeta_0)\vx_i\vx_i\trans\}\right](\breve\vbeta^{(b-1)}-\vbeta_0)\right\|_\infty\\
	\le&\,\left\|\frac{1}{\sum_{k=2}^Kn_k^{(b)}}\sum_{k=2}^K\sum_{i\in\calD_k^{(b)}}g''(\vx_i\trans\vbeta_0)\vx_i\vx_i\trans-\mathbb E\{g''(\vx_i\trans\vbeta_0)\vx_i\vx_i\trans\}\right\|_{\max}\|(\breve\vbeta^{(b-1)}-\vbeta_0)\|_1\\
	\le&\,Cs^*\sqrt{\frac{\log(p\vee N_K^{(b-1)})\,\log(p\vee\sum_{k=2}^Kn_k^{(b)})}{N_K^{(b-1)}\sum_{k=2}^Kn_k^{(b)}}}.
\end{align*}
For $I_3^{(b)}$, with probability at least $1-\exp\{-C\log(p\vee n_1^{(b)})\}$,
\begin{align*}
	I_3^{(b)}:=&\,\left\|\left[\frac{1}{n_1^{(b)}}\sum_{i\in\calD_1^{(b)}}g''(\vx_i\trans\vbeta_0)\vx_i\vx_i\trans-\mathbb E\{g''(\vx_i\trans\vbeta_0)\vx_i\vx_i\trans\}\right](\breve\vbeta^{(b-1)}-\vbeta_0)\right\|_\infty\\
	\le&\,\left\|\frac{1}{n_1^{(b)}}\sum_{i\in\calD_1^{(b)}}g''(\vx_i\trans\vbeta_0)\vx_i\vx_i\trans-\mathbb E\{g''(\vx_i\trans\vbeta_0)\vx_i\vx_i\trans\}\right\|_{\max}\|(\breve\vbeta^{(b-1)}-\vbeta_0)\|_1\\
	\le&\, Cs^*\sqrt{\frac{\log(p\vee n_1^{(b)})\,\log(p\vee N_K^{(b-1)})}{n_1^{(b)}N_K^{(b-1)}}}.
\end{align*}
For $I_4^{(b)}$, with probability at least $1-\exp\{-C\log(p\vee n_1^{(b)})\}-\exp\{-Cn_1^{(b)}\}$,
\begin{align*}
	I_4^{(b)}:=&\,\left\|\frac{1}{n_1^{(b)}}\sum_{i\in\calD_1^{(b)}}\{g''(\vx_i\trans\vbeta_0)-g''(\vx_i\trans\breve\vbeta^{(b-1)})\}\vx_i\vx_i\trans(\breve\vbeta^{(b-1)}-\vbeta_0)\right\|_\infty\\
	\le&\,CL_g\max_{i\in\calD_1^{(b)}}\|\vx_i\|_\infty\,\frac{1}{n_1^{(b)}}\sum_{i\in\calD_1^{(b)}}\{\vx_i\trans(\breve\vbeta^{(b-1)}-\vbeta_0)\}^2\\
	\le&\,\frac{Cs^*\log(p\vee N_K^{(b-1)})\sqrt{\log(p\vee n_1^{(b)})}}{N_K^{(b-1)}}.
\end{align*}
Thus, we have
\begin{align*}
	&\,\frac{1}{N_K^{(b)}}\left\|\left\{\sum_{k=2}^K\nabla^2\calL_k^{(b)}(\breve\vbeta^{(b-1)})-\frac{\sum_{k=2}^Kn_k^{(b)}}{n_1^{(b)}}\nabla^2\calL_1^{(b)}(\breve{\vbeta}^{(b-1)})\right\}(\breve{\vbeta}^{(b-1)}-\vbeta_0)\right\|_\infty\\
	\le&\,\frac{C\sum_{k=2}^Kn_k^{(b)}}{N_K^{(b)}}\cdot\frac{s^*\log(p\vee N_K^{(b-1)})\sqrt{\log(p\vee\sum_{k=1}^Kn_k^{(b)})}}{N_K^{(b-1)}}\\
	&\,+\frac{C\sum_{k=2}^Kn_k^{(b)}}{N_K^{(b)}}\cdot s^*\sqrt{\frac{\log(p\vee N_K^{(b-1)})}{N_K^{(b-1)}}}\left(\sqrt{\frac{\log(p\vee n_1^{(b)})}{n_1^{(b)}}}+\sqrt{\frac{\log(p\vee\sum_{k=2}^Kn_k^{(b)})}{\sum_{k=2}^Kn_k^{(b)}}}\right).
\end{align*}
Thus, the first line of \eqref{eq:Rb-bound-final} is $o(\lambda^{(b)})$ under the scaling condition \eqref{eq:scale2}.

 To control the second line of \eqref{eq:Rb-bound-final}, we similarly have
\begin{align*}
	&\,\left\|\left\{\frac{1}{\sum_{k=1}^Kn_k^{(j)}}\sum_{k=1}^K\nabla^2\calL_k^{(j)}(\breve\vbeta^{(j)})-\frac{1}{n_1^{(j)}}\nabla^2\calL_1^{(j)}(\breve\vbeta^{(j)})\right\}(\breve\vbeta^{(j)}-\vbeta_0)\right\|_\infty\\
	\le&\,II_1^{(j)}+II_2^{(j)}+II_3^{(j)}+II_4^{(j)}.
\end{align*}
For $II_1^{(j)}$, with probability at least $1-\exp\{-C\log(p\vee N_K^{(b-1)})\}-\sum_{j=1}^{b-1}\exp\{-C\log(p\vee\sum_{k=1}^Kn_k^{(j)})\}-\sum_{j=1}^{b-1}\exp\{-C\sum_{k=1}^Kn_k^{(j)})\}$,
\begin{align*}
	II_1^{(j)}:=&\,\left\|\frac{1}{\sum_{k=1}^Kn_k^{(j)}}\sum_{k=1}^K\sum_{i\in\calD_k^{(j)}}\{g''(\vx_i\trans\breve\vbeta^{(j)})-g''(\vx_i\trans\vbeta_0)\}\vx_i\vx_i\trans(\breve\vbeta^{(j)}-\vbeta_0)\right\|_\infty\\
	\le&\,CL_g\max_{i\le N_K^{(b-1)}}\|\vx_i\|_\infty\,\frac{1}{\sum_{k=1}^Kn_k^{(j)}}\sum_{k=1}^K\sum_{i\in\calD_k^{(j)}}\{\vx_i\trans(\breve\vbeta^{(j)}-\vbeta_0)\}^2\\
	\le&\,\frac{Cs^*\log(p\vee N_K^{(j)})\sqrt{\log(p\vee N_K^{(b-1)})}}{N_K^{(j)}},
\end{align*}
uniformly in $j\le b-1$.
For $II_2^{(j)}$, with probability at least $1-\sum_{j=1}^{b-1}\exp\{-C\log(p\vee\sum_{k=1}^Kn_k^{(j)})\}$,
\begin{align*}
	II_2^{(j)}:=&\,\left\|\left[\frac{1}{\sum_{k=1}^Kn_k^{(j)}}\sum_{k=1}^K\sum_{i\in\calD_k^{(j)}}g''(\vx_i\trans\vbeta_0)\vx_i\vx_i\trans-\mathbb E\{g''(\vx_i\trans\vbeta_0)\vx_i\vx_i\trans\}\right](\breve\vbeta^{(j)}-\vbeta_0)\right\|_\infty\\
	\le&\,\left\|\frac{1}{\sum_{k=1}^Kn_k^{(j)}}\sum_{k=1}^K\sum_{i\in\calD_k^{(j)}}g''(\vx_i\trans\vbeta_0)\vx_i\vx_i\trans-\mathbb E\{g''(\vx_i\trans\vbeta_0)\vx_i\vx_i\trans\}\right\|_{\max}\|\breve\vbeta^{(j)}-\vbeta_0\|_1\\
	\le&\,Cs^*\sqrt{\frac{\log(p\vee N_K^{(j)})\,\log(p\vee\sum_{k=1}^Kn_k^{(j)})}{N_K^{(j)}\sum_{k=1}^Kn_k^{(j)}}},
\end{align*}
uniformly in $j\le b-1$.
For $II_3^{(j)}$, with probability at least $1-\sum_{j=1}^{b-1}\exp\{-C\log(p\vee n_1^{(j)})\}$,
\begin{align*}
	II_3^{(j)}:=&\,\left\|\left\{\frac{1}{n_1^{(j)}}\sum_{i\in\calD_1^{(j)}}g''(\vx_i\trans\vbeta_0)\vx_i\vx_i\trans-\mathbb E\{g''(\vx_i\trans\vbeta_0)\vx_i\vx_i\trans\}\right\}(\breve\vbeta^{(j)}-\vbeta_0)\right\|_\infty\\
	\le&\,\left\|\frac{1}{n_1^{(j)}}\sum_{i\in\calD_1^{(j)}}g''(\vx_i\trans\vbeta_0)\vx_i\vx_i\trans-\mathbb E\{g''(\vx_i\trans\vbeta_0)\vx_i\vx_i\trans\}\right\|_{\max}\|\breve\vbeta^{(j)}-\vbeta_0\|_1\\
	\le&\,Cs^*\sqrt{\frac{\log(p\vee n_1^{(j)})\,\log(p\vee N_K^{(j)})}{n_1^{(j)}N_K^{(j)}}},
\end{align*}
uniformly in $j\le b-1$.
For $II_4^{(j)}$, with probability at least $1-\sum_{j=1}^{b-1}\exp\{-C\log(p\vee n_1^{(j)})\}-\sum_{j=1}^{b-1}\exp\{-Cn_1^{(j)}\}$,
\begin{align*}
	II_4^{(j)}:=&\,\left\|\frac{1}{n_1^{(j)}}\sum_{i\in\calD_1^{(j)}}\{g''(\vx_i\trans\breve\vbeta^{(j)})-g''(\vx_i\trans\vbeta_0)\}\vx_i\vx_i\trans(\breve\vbeta^{(j)}-\vbeta_0)\right\|_\infty\\
	\le&\,CL_g\max_{i\le n_1^{(j)}}\|\vx_i\|_\infty\frac{1}{n_1^{(j)}}\,\sum_{i\in\calD_1^{(j)}}\{\vx_i\trans(\breve\vbeta^{(j)}-\vbeta_0)\}^2\\
	\le&\,\frac{Cs^*\log(p\vee N_K^{(j)})\sqrt{\log(p\vee n_1^{(j)})}}{N_K^{(j)}}.
\end{align*}
Thus, we conclude that
\begin{align*}
	&\,\frac{1}{N_K^{(b)}}\left\|\sum_{j=1}^{b-1}\left\{\sum_{k=1}^K\nabla^2\calL_k^{(j)}(\breve\vbeta^{(j)})-\frac{\sum_{k=1}^Kn_k^{(j)}}{n_1^{(j)}}\nabla^2\calL_1^{(j)}(\breve\vbeta^{(j)})\right\}(\breve\vbeta^{(j)}-\vbeta_0)\right\|_\infty\\
	\le&\,\frac{Cs^*}{N_K^{(b)}}\sum_{j=1}^{b-1}\sum_{k=1}^Kn_k^{(j)}\left\{\sqrt{\frac{\log(p\vee n_1^{(j)})\,\log(p\vee N_K^{(j)})}{n_1^{(j)}N_K^{(j)}}}+\frac{\log(p\vee N_K^{(j)})\sqrt{\log(p\vee N_K^{(b-1)})}}{N_K^{(j)}}\right\}\\
	\le&\,\frac{Cs^*\sqrt{\log(p\vee N_K^{(b)})}}{N_K^{(b)}}\cdot\max_{1\le j\le b-1}\sqrt{\frac{\log(p\vee n_1^{(j)})}{n_1^{(j)}}}\cdot\sum_{j=1}^{b-1}\sum_{k=1}^K\frac{n_k^{(j)}}{\sqrt{N_K^{(j)}}}\\
	&\,+\frac{Cs^*\log^{3/2}(p\vee N_K^{(b)})}{N_K^{(b)}}\sum_{j=1}^{b-1}\sum_{k=1}^K\frac{n_k^{(j)}}{N_K^{(j)}}\\
	\le&\,Cs^*\sqrt{\frac{\log(p\vee N_K^{(b)})}{N_K^{(b)}}}\max_{1\le j\le b-1}\sqrt{\frac{\log(p\vee n_1^{(j)})}{n_1^{(j)}}}+\frac{Cs^*\log^{3/2}(p\vee N_K^{(b)})(1+\log(N_K^{(b-1)}/N_K^{(1)}))}{N_K^{(b)}},
\end{align*}
where the last inequality is implied by the fact that $\sum_{j=1}^{b-1}\sum_{k=1}^Kn_k^{(j)}/N_K^{(j)}\le1+\log(N_K^{(b-1)}/N_K^{(1)})$ and $\sum_{j=1}^{b-1}\sum_{k=1}^Kn_k^{(j)}/\sqrt{N_K^{(j)}}\le2\sqrt{N_K^{(b-1)}}$ in Lemma 2 of \cite{luo2023online}.
Thus, the second line of \eqref{eq:Rb-bound-final} is $o(\lambda^{(b)})$ under the scaling condition \eqref{eq:scale2}. 
Therefore, $\|\vR^{(b)}\|_\infty$ is bounded by $o(\lambda^{(b)})$.

For a sufficiently large constant $C$, combining \eqref{eq:score-decomp-b1}--\eqref{eq:score-emp} with the bound on $\|\vR^{(b)}\|_\infty$ implies \eqref{eq:dist-score}.

\noindent
\textbf{Step 4 (conclude the lasso rates).}
Using similar arguments as in the proof of Theorem~\ref{thm:online-lasso} (Step 4), we obtain
\begin{equation*}
	\kappa_I\|\vDelta_{\zeta_b}^{(b)}\|_2^2\le D^{(b)}(\vbeta_0+\vDelta_{\zeta_b}^{(b)},\vbeta_0)\le\frac{3}{2}\lambda^{(b)}\,\|\vDelta_{\zeta_b}^{(b)}\|_1\le\frac{3}{2}\lambda^{(b)}\cdot4\sqrt{s^*}\,\|\vDelta_{\zeta_b}^{(b)}\|_2.
\end{equation*}
Thus $\|\vDelta_{\zeta_b}^{(b)}\|_2\le C\sqrt{s^*}\,\lambda^{(b)}$.
The $\ell_1$ bound follows $\|\vDelta_{\zeta_b}^{(b)}\|_1\le 4\sqrt{s^*}\,\|\vDelta_{\zeta_b}^{(b)}\|_2\le Cs^*\lambda^{(b)}$.
Using the contradiction argument as in the proof of Theorem~\ref{thm:online-lasso} (Step 5), the error bounds also hold for $\|\vDelta^{(b)}\|$.
\hfill $\Box$

\end{document}

%% file: self-define-Heng.tex
\newcommand{\va}{{\bf a}}

\newcommand{\vu}{{\bf u}}

\newcommand{\vx}{{\bf x}}

\newcommand{\vA}{{\bf A}}

\newcommand{\vI}{{\bf I}}

\newcommand{\vR}{{\bf R}}
\newcommand{\vT}{{\bf T}}

\newcommand{\vX}{{\bf X}}

\newcommand{\vDelta}{\mbox{\boldmath $\Delta$}}

\newcommand{\vSigma}{\mbox{\boldmath $\Sigma$}}

\newcommand{\vbeta}{\mbox{\boldmath $\beta$}}

\newcommand{\vzeta}{\mbox{\boldmath $\zeta$}}

\newcommand{\vxi}{\mbox{\boldmath $\xi$}}
\newcommand{\ep}{\epsilon}

\newcommand{\bay}{\begin{array}}
\newcommand{\eay}{\end{array}}

\newcommand{\bqa}{\begin{eqnarray*}}
\newcommand{\eqa}{\end{eqnarray*}}
\newcommand{\bqan}{\begin{eqnarray}}
\newcommand{\eqan}{\end{eqnarray}}
\newcommand{\bqt}{\begin{quote}}
\newcommand{\eqt}{\end{quote}}
\newcommand{\bt}{\begin{tabbing}}
\newcommand{\et}{\end{tabbing}}
\newcommand{\bit}{\begin{itemize}}
\newcommand{\eit}{\end{itemize}}
\newcommand{\ben}{\begin{enumerate}}
\newcommand{\een}{\end{enumerate}}
\newcommand{\beq}{\begin{equation}}
\newcommand{\eeq}{\end{equation}}
\newcommand{\bdefi}{\begin{definition}}
\newcommand{\edefi}{\end{definition}}
\newcommand{\bpro}{\begin{proposition}}
\newcommand{\epro}{\end{proposition}}
\newcommand{\bco}{\begin{corollary}}
\newcommand{\eco}{\end{corollary}}
\newcommand{\bdes}{\begin{description}}
\newcommand{\edes}{\end{description}}

%% file: self-define-HL.tex
\def\log{\hbox{log}}

\def\boxit#1{\vbox{\hrule\hbox{\vrule\kern6pt
          \vbox{\kern6pt#1\kern6pt}\kern6pt\vrule}\hrule}}

\def\bse{\begin{eqnarray*}}
\def\ese{\end{eqnarray*}}
\def\be{\begin{eqnarray}}
\def\ee{\end{eqnarray}}
\def\bq{\begin{equation}}
\def\eq{\end{equation}}

\def\trans{^{\top}}

\newtheorem{proposition}{Proposition}

\newcommand{\blem}{\begin{lemma}}
\newcommand{\elem}{\end{lemma}}
\newcommand{\bthe}{\begin{theorem}}
\newcommand{\ethe}{\end{theorem}}

\newtheorem{definition}{Definition}[section]
\newtheorem{lemma}[definition]{Lemma}
\newtheorem{theorem}[definition]{Theorem}

\def\delete#1{\iffalse #1 \fi}

\def\bse{\begin{eqnarray*}}
\def\ese{\end{eqnarray*}}
\def\bee{\begin{enumerate}}
\def\eee{\end{enumerate}}
\def\bqe{\begin{eqnarray}}
\def\eqe{\end{eqnarray}}
\def\bed{\begin{description}}
\def\eed{\end{description}}
\def\bei{\begin{itemize}}
\def\eei{\end{itemize}}

\def\pmb#1{\setbox0=\hbox{#1}%
    \kern-.025em\copy0\kern-\wd0
    \kern.05em\copy0\kern-\wd0
    \kern-.025em\raise.0433em\box0 }
\def\pmbh#1#2{\setbox0=\hbox{#1}%
    \setbox1=\hbox{#2}%
    \kern-.025em\copy0\kern-\wd0
    \kern.05em\copy1\kern-\wd0
    \kern-.025em\raise.0433em\box0 }

\def\frac#1#2{{#1\over#2}}

\def\boxit#1{\vbox{\hrule\hbox{\vrule\kern6pt
   \vbox{\kern6pt#1\kern6pt}\kern6pt\vrule}\hrule}}

\def\listing#1{\vskip 4mm\begin{verbatim}\input#1 \vskip 4mm}
\def\thick#1{\hbox{\rlap{$#1$}\kern0.25pt\rlap{$#1$}\kern0.25pt$#1$}}

\def\pmbh{{\pmb h}}

\def\calA{{\cal A}}

\def\calD{{\cal D}}
\def\calE{{\cal E}}

\def\calL{{\cal L}}

\def\calS{{\cal S}}

\renewcommand\today{\ifcase\month\or
   Jan\or Feb\or Mar\or Apr\or May\or
   Jun\or Jul\or Aug\or Sep\or Oct\or Nov\or
   Dec\fi
   \space\number\day, \number\year}
